\documentclass[11pt,article]{article}
\usepackage[left=2cm,right=2cm]{geometry}

\usepackage{natbib}

\usepackage{hyperref}
\usepackage{url}
\usepackage{microtype}
\usepackage{graphicx}
\usepackage{subfigure}
\usepackage{booktabs} 

\usepackage{amsmath,amsfonts,bm}

\def\eqref#1{equation~\ref{#1}}

\def\1{\bm{1}}

\DeclareMathAlphabet{\mathsfit}{\encodingdefault}{\sfdefault}{m}{sl}
\SetMathAlphabet{\mathsfit}{bold}{\encodingdefault}{\sfdefault}{bx}{n}

\DeclareMathOperator*{\argmin}{arg\,min}

\usepackage{url}

\usepackage{microtype}
\usepackage{graphicx}
\usepackage{subfigure}
\usepackage{balance}
\usepackage{booktabs} 
\usepackage{amssymb}
\usepackage{amsthm}
\usepackage{array}

\usepackage{graphicx}
\usepackage{clrscode}
\usepackage{subfigure}
\usepackage{multirow}
\usepackage{multicol}
\usepackage{float}
\usepackage{color}
\usepackage{xcolor}
\usepackage{amsopn}
\usepackage{mathtools}
\usepackage{booktabs}
\usepackage{arydshln}
\usepackage{hyperref}
\usepackage{blkarray}
\usepackage{enumerate}
\usepackage{courier}
\usepackage{rotating}
\usepackage{bm}
\usepackage{subfigure}
\usepackage{array}
\usepackage{ragged2e}
\usepackage{hyperref}
\usepackage{amsmath}
\usepackage{threeparttable} 

\usepackage[ruled,linesnumbered]{algorithm2e}

\usepackage{ulem}
\usepackage{tikz}
\newcommand*{\circled}[1]{\lower.7ex\hbox{\tikz\draw (0pt, 0pt)%
    circle (.5em) node {\makebox[1em][c]{\small #1}};}}

\SetKwComment{comment}{ $triangleright$ \ }{}
\theoremstyle{definition}
\newtheorem{definition}{Definition}
\theoremstyle{lemma}
\newtheorem{lemma}{Lemma}
\theoremstyle{theorem}
\newtheorem{theorem}{Theorem}
\theoremstyle{claim}

\theoremstyle{remark}
\newtheorem*{remark}{Remark}

\begin{document}
\title{Generalizable Information Theoretic Causal Representation}
\author{Mengyue Yang$^{1}$\footnotemark[1], Xinyu Cai$^{4}$, Furui Liu$^{2}$, Xu Chen$^{3}$, Zhitang Chen$^{2}$, Jianye Hao$^{2}$, Jun Wang$^{1}$ \\ $^1$University College London $^2$Noah's Ark Lab, Huawei \\ $^3$Renmin University of China $^4$ Nanyang Technological University}
\renewcommand{\thefootnote}{\fnsymbol{footnote}}
\footnotetext[1]{Email to: \url{mengyue.yang.20@ucl.ac.uk}}

\maketitle
\begin{abstract}
It is evidence that representation learning can improve model's performance over multiple downstream tasks in many real-world scenarios, such as image classification and recommender systems. Existing learning approaches rely on establishing the correlation (or its proxy) between features and the downstream task (labels), which typically  results in a representation containing cause, effect and spurious correlated variables of the label. 
Its generalizability may deteriorate because of the unstability of the non-causal parts. 
In this paper, we propose to learn causal representation from observational data by regularizing the learning procedure with mutual information measures according to our hypothetical causal graph. The optimization involves a counterfactual loss, based on which we deduce a theoretical guarantee that the causality-inspired learning is with reduced sample complexity and better generalization ability. Extensive experiments show that the models trained on causal representations learned by our approach is  robust under adversarial attacks and distribution shift. 

\end{abstract}

\section{Introduction}
Learning representations from purely observations concerns the problem of finding a low-dimensional, compact representation which is beneficial to  prediction models for multiple downstream tasks. It is widely applied in many real-world applications like recommendation system, searching system etc.\citep{sun2018recurrent, okura2017embedding, zhang2017joint, shi2018heterogeneous}. Generally, the feature representation learning is to map the collected features into a representation space, which is then leveraged for training a prediction model. Most of the representation learning methods are designed based on the information of correlation between the feature and downstream labels only. Some recent works \citep{scholkopf2021toward, peters2017elements, zhou2021domain} reckon that the representation produced by this schema is not able to achieve robustness and generalization, since spurious correlated features miss essential information which actually influence the system \citep{pearl2009causality}. 
For example, obtaining thermometer enables the prediction of the air temperature, but manually putting the thermometer into hot water does not change the air temperature. It is obvious that thermometer scale is not always a stable feature to predict temperature, because non-causal relations exist. Therefore, the causal information rather than correlation is more robust and general for prediction models.





Causal representation learning is an effective approach for extracting invariant, cross-domain stable causal information, which is believed to be able to improve sample efficiency by understanding the underlying generative mechanism from observational data \citep{scholkopf2021toward, DBLP:journals/advcs/AyP08}. Recently, some approaches learn the causal representation making use of certain causal property provided by the specified priors. One perspective is to learn feature representations with its structure specified by causal models, assuming that the raw data contains both cause variables and effect variables, like causal disentangled representation learning form images \citep{trauble2020independence, yang2021causalvae, shen2020disentangled, suter2019robustly, dittadi2020transfer} and physical environments \citep{ahmed2020causalworld, sontakke2021causal}. The representations have good interpretability, but probably suffer from redundancy because the learning process may generate some parts unrelated to the downstream tasks. Another perspective is to utilize the causal relation between features and labels \citep{suter2019robustly, kilbertus2018generalization, scholkopf2012causal, wang2021desiderata}. They are interpreted as causal or anti-causal learning, tackling cases such as the causal direction between observational feature set $\mathbf{X}$ and downstream labels $\mathbf{Y}$ is $\mathbf{X}\rightarrow \mathbf{Y}$ and finding the cause information for downstream prediction. This only covers a set of simple cases like handwritten numeral recognition, where the  handwritten numeral $\mathbf{X}$ is produced by the number $\mathbf{Y}$ in brains. However, its effectivity is not guaranteed in circumstances where causal graph is complex and unknown previously.

\begin{figure*}[t]
\begin{center}
\centerline{\includegraphics[width=1\columnwidth]{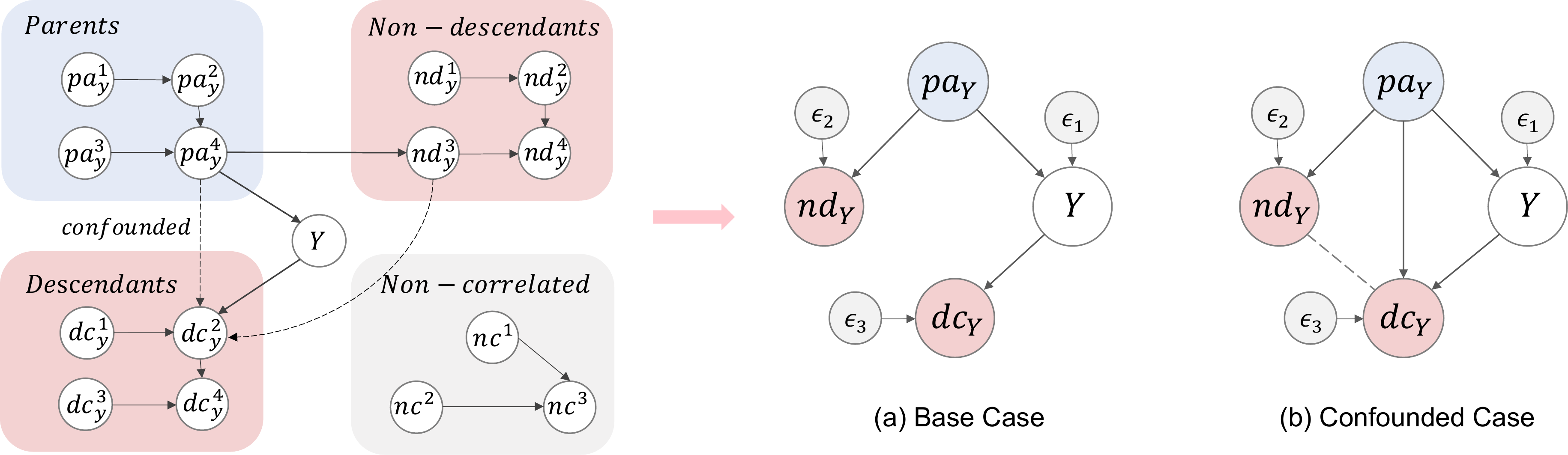}}
\caption{The figure demonstrates a typical case of a causal system, and we extend to two more general settings: (a) the causal graph without confounders, (b) the causal graph with confounders.}
\label{fig:intro}
\end{center}
\vspace{-5mm}
\end{figure*}

In this paper, we formalize the real-world causal system as Fig.\ref{fig:intro} shows. The observational data $\mathbf{X} = \mathbf{h}(\mathbf{pa_Y}, \mathbf{nd_Y},  \mathbf{dc_Y},\mathbf{nc_Y})$ is generated by a mixture of the minimal sufficient parent nodes $\mathbf{pa_Y}$, non-descendants $\mathbf{nd_Y}$, descendants $\mathbf{dc_Y}$ and uncorrelated features $\mathbf{nc_Y}$ of $\mathbf{Y}$. In general, we treat the causal system from information theoretic perspective by considering the natural data generative process as an information propagation along the causal graph. We deal with the general problem of learning minimal sufficient cause information $\mathbf{pa_Y}$ for downstream predictions. {To learn a good representation $\mathbf{Z=pa_Y}$ from $\mathbf{X}$, we propose an approach named \textbf{CaRR} that consists of  two parts: basic mutual information objective and counterfactual estimation of probability of necessary and sufficient (PNS).} 
As for the first part, we derive a tractable mutual information causal representation learning objective from the perspective of Data Processing Inequality (DPI).
About the second part, we introduce an objective to find a robust {representation of} $\mathbf{pa_Y}$. We take a perspective of adversarial attack, and propose an counterfactual vulnerable (CV) term to quantify the necessary and sufficient representation of causal information. In addition, we theoretically analyze the generalization ability under finite sample by information theory, and provide a condition that $\mathbf{Z}=\mathbf{pa_Y}$ is the minimizer of the finite-sample upper bound of a downstream prediction error over all possible solutions in the space of $\mathbf{Z}$. Our theory provides an insight that sample efficiency will be improved if causal information is obtained.

\textbf{Our Contribution}

(i) We propose an information-theoretical approach to learn causal representation from observational data, in the formalization of an explicit causal graphical model to describe the data generative process of the  real-world system.

(ii) We propose a novel quantification of the causal effect of the representation on the downstream labels by measuring the interventional vunlerability, based on which a robust learning approach is proposed correspondingly to optimize our model for causality-inspired learning. This integrates the concept of intervention in causality with the domain of representation learning. 

(iii) We theoretically analyze the sample efficiency of the learning approach by giving a generalization error bound with respect to finite sample size. The theorem depicts a quantitative link between the amount of causal information contained in the learned representation, and the sample complexity of the model on downstream tasks.

(iv) Comprehensive experiments to verify the merits of method are conducted, including testing cases for the model's generalization ability when adversarial attack in representation space  and distribution shift on dataset exist.
\section{Related Works}
Causal Representation Learning is a set of approaches to find reusable causal information and causal mechanism from observational data. Motivated by learning generalizable and stable information from data,  causal representation learning approaches from several different perspectives have been proposed in literature.  Under the framework of corss-domain learning, the pioneering work \citep{zhou2021domain, wang2021desiderata,shen2021towards} consider the heterogeneity across multiple domains under the out-of-distribution settings \citep{gong2016domain, li2018domain, magliacane2017domain, zhang2015multi}. They learn causal representations from observational data by invariant causal mechanism across multi-domains. Leveraging and combining the idea of causal structure learning,  some works use structural causal models to describe causal relationship inside the mixed observational data \citep{ yang2021causalvae, shen2020disentangled} and perform learning by minimizing a loss containing a structure learning part. On the other hand, based on the recently proposed independence between cause and mechanism principle (ICM), several work focus on the 
assymmetry between cause and effect. \citep{sontakke2021causal, steudel2010causal} use the asymmetrical Komolgorov Complexity \citep{janzing2010causal, cover1999elements} relationship between cause and effect for causal learning, and similar ideas are utilized by \citep{parascandolo2018learning, steudel2010causal}. Different from ICM, our paper considers the relationship between cause and effect from the perspective of mutual information \cite{DBLP:conf/icml/BelghaziBROBHC18, DBLP:conf/icml/ChengHDLGC20} since cause information will reduce along the causal generative process (i.e., a phenomenon known as data processing inequality \citep{kullback1997information, cover1999elements}). We put our attentions on the generalization ability of the causal representation. Most existing causal inference in generalization focuses on distribution shift setting \citep{vapnik1999overview, rojas2018invariant, meinshausen2018causality, peters2016causal}. \cite{arjovsky2019invariant} aim at finding representation containing invariant information and analyze the generalization ability of method. Different with them, in our paper, we theoretically consider the generalization ability of causal representation under finite sample in probability approximate correctly view (PAC) \citep{shalev2014understanding, shamir2010learning}, and demonstrate that causal representation has low generalization error, which also supports the claim that causal information is sample efficient.



\section{Preliminaries}
\textbf{Structural Causal Models}.\citep{pearl2009causality}
In causality, the causal graphical model is represented quantitatively by a functional model named Structural Causal Models (SCMs). The variable $Y$ is represented as a function of its parent variable $\mathbf{X}$, and additional noise $\bm{\epsilon}$,
$$ Y =  f(\mathbf{X},\bm{\epsilon}), \mathbf{X}\perp\bm{\epsilon}$$
Under different assumptions, the model can be estimated from observational data up to certain degree by linear and nonlinear regression, or independence test based approaches. The information propogation is described by the functional models, which also enables the estimation of causal effect as well and causal direction under the graphical model. 

\textbf{Counterfactual Estimation}.\citep{pearl2009causality}
Pearl's causality gives a 3-step approach -- 'abduction-action-prediction', for counterfactual estimation. In first abduction step, let $\mathbf{O}=\{\mathbf{X}, Y\}$ denote observational data, $\mathbf{X}$ is the cause of $Y$, we should infer the posterior of the exogenous variable $q(\bm{\epsilon}|\mathbf{O})$. Action approach is implemented by the language of do-operations, where $do(\mathbf{X=x'})$ sets the variable to be a value of $\mathbf{x'}$ and the prediction approach is to observe the change of the output under such action $p(y|\mathbf{x'}, \bm{\epsilon})$ . The counterfactual result is thus the interventional output specified by the SCMs. The information propogation along the causal path results in the observed counterfactual predictions which are later used as a base for quantifying the causality inside the system.   

\textbf{Mutual Information}.\citep{cover1999elements} We explore the causal characteristics on mutual information level. Mutual information is an entropy based measure of the mutual dependence between variables.
\begin{definition}The mutual information between two random variables $\mathbf{X, Z}$ is define
\begin{equation}
    {\small{{I}(\mathbf{X; Z})=\int_{\mathcal{Z}} \int_{\mathcal{X}} p_{X Z}(\boldsymbol{x}, \boldsymbol{z}) \log \left(\frac{p_{X Z}(\boldsymbol{x}, \boldsymbol{z})}{p_{X}(\boldsymbol{x}) p_{Z}(\boldsymbol{z})}\right) d \boldsymbol{x} d \boldsymbol{z}}}
\end{equation}
\end{definition}
\textbf{Data Processing Inequality (DPI)}.\citep{cover1999elements}  DPI is an information theoretic concept which describes the decreasing of the information along the Markov chain.
\begin{definition}(Data Processing Inequality) If three random variables form the Markov chain $\mathbf{X-Z-Y}$, there exist an inequality:
\begin{equation}{}
    {\small{{I}(\mathbf{X ; Z})\ge{I}(\mathbf{X ; Y})}}
\end{equation}
\end{definition}
\section{Method}
In this section, we demonstrate a method to specify the minimal sufficient parents information $\mathbf{pa_Y}$ from mixed observation $\mathbf{X}$. We firstly analyze the information propogation among different causal variables under two typical causal graphs, based on which we propose an objective function with mutual information constraint for casual representation learning. To optimize the model under such function to fulfill our hypothetical causal structure, an counterfactual vulnerability based approach is also introduced.

\subsection{Problem Definition}
\label{sec:main_obj}
Considering complex scenario like recommendation system and fault detection system. Fig.\ref{fig:intro} (left) describes the relations among observations and predicted labels, which can be generalized to the setting of Fig.\ref{fig:intro} (a)(b), {including the  base case and its extended confounding case. }The only difference lies on the additional edge from $\mathbf{pa_Y}$ to $\mathbf{dc_Y}$, since some nodes in $\mathbf{pa_Y}$ will direct point to the nodes in $\mathbf{dc_Y}$, which makes  $\mathbf{pa_Y}$ confound $\mathbf{dc_Y}$ and $Y$.
In base case, no edge between $\mathbf{nd_Y}$ and $Y$ exists, otherwise one can make adjustments by merging $\mathbf{nd_Y}$ into $\mathbf{pa_Y}$ or $\mathbf{dc_Y}$. In the confounded case, {dash lines in Fig.\ref{fig:intro} (left) exist, which means $\mathbf{pa_Y}$ contains the confounder between $Y$, $\mathbf{dc_Y}$ and $\mathbf{nd_Y}$}. Denote $\mathbf{X}\in\mathcal{X}$ as $d$-dimensional observational data like context information or features in real-world system, and ${Y}\in\mathcal{Y}$ as the labels of downstream tasks. Each pair of sample $(\mathbf{x,} y)$ is drawn i.i.d. from joint distribution $p(\mathbf{x},y)$. We use $\mathbf{pa_Y}\in\mathbb{R}^{p_1}$ to denote the variables including all observable parent nodes of $Y$ in the causal graph. Similarly, $\mathbf{dc_Y}\in\mathbb{R}^{p_2}$ and $\mathbf{nd_Y}\in\mathbb{R}^{p_3}$ denote the descendant and non-descendant nodes of $\mathbf{Y}$, respectively. 
{The minimum sufficient cause information $\mathbf{pa_Y}$ in causal system is stable information for predicting $y$. To specify $\mathbf{pa_Y}$ from observational data $\mathbf{X}$, we firstly consider the intrinsic causal mechanism behind observational data, whose causal graph is formulated by Structure Causal Models (SCMs) as}

\begin{equation}\label{eq:scms}
\setlength\abovedisplayskip{1pt}
\setlength\belowdisplayskip{5pt}
\left\{
\begin{aligned}
&\mathbf{Y} = \mathbf{g_1}(\mathbf{pa_Y}, \bm{\epsilon}_1),\\
&\mathbf{nd_Y} = \mathbf{g_2(pa_Y}, \bm{\epsilon}_2), \\ &\mathbf{dc_Y} = {\mathbf{g_3}(Y}, \bm{\epsilon_3}) \\
&\mathbf{pa_Y}\perp\bm{\epsilon_1},\mathbf{pa_Y}\perp\bm{\epsilon_2}, {Y}\perp{\bm{\epsilon_3}},
\end{aligned}
\right.
\end{equation}

where $\bm{\epsilon_1}$, $\bm{\epsilon_2}$ and $\bm{\epsilon_3}$ are assumed to be Gaussian noise, with a distribution $\mathcal{N}(\mathbf{0, \beta\mathbf{I}})$. 
{The SCMs help us to analyze the information flow among the variables. Different from classic supervise learning method, we consider the relationship among $\mathbf{pa_Y}$, $\mathbf{dc_Y}$ and $\mathbf{nd_Y}$ instead of considering that between $\mathbf{X}$ and $\mathbf{Y}$ only. }

\subsection{Causal Representation Learning Based on Mutual Information}
The difficulty lies on distinguishing $\mathbf{pa_Y}$ from $\mathbf{X}$, when one only observes $\mathbf{X}$, a mixture of them. 
{A fact is that the information contained in root $\mathbf{pa_Y}$ suffers from degradation along  causal paths, based on which we deduce inequalities on the mutual information between node pairs and design a method to learn the $\mathbf{pa_Y}$ from observational data. We obtain the following Lemma which describes the  relationship between different variables concerning mutual information. It is the key to help us identify the information of $\mathbf{pa_Y}$.}
\begin{lemma}\label{lem:dpi}
Considering two scenarios described by Fig.\ref{fig:intro}(a)(b), the following inequality is held for both base and confounded cases:
\begin{equation}
\begin{split}
\label{eq:dpi}
    I(\mathbf{pa_Y; nd_Y, dc_Y}) \le I(\mathbf{pa_Y; nd_Y}, Y)\\
\end{split}
\end{equation}
\end{lemma}
{Note that the equality is achieved when SCMs are deterministic in causal system and the variance of exogenous variables is equal to zero. In other words, it is a  deterministic generative system.}
We develop an algorithm to learn representations based on such hypothetical structure using the presented inequalities.  Let  $\mathbf{Z} = \phi(\mathbf{X})$ denote representation extracted from original observation $\mathbf{X}$, where $\phi: \mathcal{X}\rightarrow\mathcal{Z}$ is the representation extraction function. For the causal system shown in Fig.\ref{fig:intro}, an important fact is that $\mathbf{pa_Y}$ is the minimal sufficient statistics of the observational data and $\mathbf{Y}$, since it is the root cause that dominates the generative process of causal system defined in Fig.\ref{fig:intro}. If there exists a mapping from $\mathbf{X}$ to $\mathbf{pa_Y}$, it is a function that finds the minimal sufficient statistics of the causal system. The minimal sufficient is formally defined as follows:

\begin{definition}
(Minimal Sufficient Statistic \citep{lehmann2012completeness}). Let $\mathbf{X}, Y$ be random variables.  $\mathbf{Z'}$ is sufficient for ${Y}$ if and only if $\forall \mathbf{x \in \mathcal{X}}, \mathbf{z' \in \mathcal{Z}}, {y \in \mathcal{Y}}, p(\mathbf{x|z',} y) = p(\mathbf{x|z'})$. A sufficient statistic $\mathbf{Z^*}$ is minimal if and only if for any sufficient statistic $\mathbf{Z}$, there exists a deterministic function f such that $\mathbf{Z^* = f(Z)}$ almost everywhere w.r.t $\mathbf{X}$.
\end{definition} 
In this paper, we argue that a satisfactory solution of the representation is that $\mathbf{Z}$ is equal to $\mathbf{pa_Y}$.  To optimize the model under such objective based on the Eq.\ref{eq:dpi}, we formulate this as a minmax optimization problem, and provide  theoretical analyses for such approach. The following theorem, proven in Appendix \ref{sup:proof}, illustrates the equivalence between satisfying Lemma \ref{lem:dpi} and  obtaining minimal sufficient statistics. 
\begin{theorem}\label{thm:sufficient_statistics}
Let $\mathbf{Z'}, \mathbf{Z}\in\mathcal{Z}$, $\mathbf{Z}=\phi(\mathbf{X})$ is minimal sufficient statistics of $(Y, \mathbf{nd_Y)}$ if and only if
\begin{equation}
    \mathbf{pa_Y} = \argmin_{\mathbf{Z}} I(\mathbf{Z;nd_Y, dc_Y})  
    \text{s.t.} I(\mathbf{Z}; Y, \mathbf{nd_Y})={\max}_{\mathbf{Z'}} I(\mathbf{Z'}; Y, \mathbf{nd_Y})
\end{equation}
\end{theorem}
By Theorem \ref{thm:sufficient_statistics}, when the underlying causal information is unavailable from observational data, we use above minmax approach to identify the causal information. The process of finding an optimal representation $\mathbf{z} = \phi(\mathbf{x})$ is alternatively formulated as maximizing the following Lagrangian function:
\begin{equation}\label{eq:delta}
\begin{split}
    \delta(\phi)  &= 
    \max_\phi \underbrace{I(\phi(\mathbf{X});Y, \mathbf{nd_Y})}_{\small{\circled{1}}}  - \lambda\underbrace{I(\phi(\mathbf{X);nd_Y, dc_Y})}_{\small{\circled{2}}}
\end{split}
\end{equation}
Note that since the information of $\mathbf{nd_Y}$ and $\mathbf{dc_Y}$ is not revealed, the above objective function is not able to be optimized directly. To get an tractable form of this objective function, we firstly present the inequalities below (proven in Appendix \ref{sup:proof}).

\begin{lemma}\label{lem:ine_i}
Suppose the features and labels are $\mathbf{X}, Y$, where $\mathbf{X}$ is generated by the consists of the minimal sufficient parents, descendants and non-descendants as $\mathbf{X} = \mathbf{h(pa_Y, nd_Y, dc_Y)}$, where $\mathbf{h}(\cdot)$ is an injective function. The following inequality holds
\begin{enumerate}
\item $I(\mathbf{pa_Y;nd_Y, dc_Y})\le I(\mathbf{pa_Y;X})$
\item $I(\mathbf{pa_Y;}Y)\le I(\mathbf{pa_Y;nd_Y},Y)$
\end{enumerate}
\end{lemma}

The same as Eq.\ref{eq:dpi}, the equality is achieved when all the functions in causal system is deterministic and $H(\bm{\epsilon})=0$. We use Lemma \ref{lem:ine_i} (1) and (2) to substitute(1)and(2)in Eq.\ref{eq:delta} respectively, to get the transformed lower bound of the objective function. The new objective function is tractable since one no longer needs to specify $\mathbf{dc_Y}$ and $\mathbf{nd_Y}$ in advance.
\begin{equation}\label{eq:ib}
    \delta(\phi) \ge L(\phi) = \max_\phi I(\phi(\mathbf{X});{Y})-\lambda I(\phi(\mathbf{X});\mathbf{X}) 
\end{equation}
The above objective function coincides with deep Information Bottleneck (IB). The difference is that IB is deduced from Rate Distortion Theorem in information theory, and it holds under the structure of Markov Chain instead of a causal graph (i.e. Fig.\ref{fig:intro}). In this paper, the IB setting is generalized into causal space, by bridging minimal sufficient statistics with root cause variables in the hypothetical causal graph. Although the tractable objective (Eq.\ref{eq:ib}) enables the optimization of the model with the function $\phi$, one challenge is that it cannot correctly discern $\mathbf{pa_Y}$ and $\mathbf{dc_Y}$ like intractable objective Eq.\ref{eq:delta}. Since $\mathbf{dc_Y}$ and $\mathbf{pa_Y}$ are both closely correlated to $\mathbf{Y}$, the learned $\phi(\mathbf{X})$ might include the  information of $\mathbf{dc_Y}$. In the following section, we apply the counterfactual estimation process to get sufficient and necessary cause $\mathbf{pa_Y}$, so that the issue is addressed.



\subsection{Learning Representation by Counterfactual Estimation of PNS}
In this section, we introduce a method to learn the minimal sufficient parental information of the labels. The core idea is to consider the counterfactual identifiable probability of necessary and sufficient (PNS) cause of predicted labels \citep{pearl2009causality, wang2021desiderata}. Under this schema, obtaining the cause information is transformed into a tractable task satisfying counterfactual PNS on $Y$.  We design a robust intervention vulnurability term to quantify the degree of satisfied PNS, and obtain a set of parameters that maximize the PNS probability, in order to find sufficient and necessary causal information.

\begin{definition}\label{def:PNS}(PNS Distribution \citet{pearl2009causality}) Suppose we observe a data point with representation $\mathbf{Z = z}$ and label ${Y = y}$. The probability of necessary and sufficient (pns) of $\mathbb{I}\{\mathbf{Z = z}\}$ for
$\mathbb{I}\{{Y = y}\}$:
\begin{equation}\label{eq:pns}
\begin{aligned}
\mathrm{PNS} &=\underbrace{\mathbb{E}_{p(\mathbf{Z} \neq \mathbf{z}, Y \neq y)} P(Y(\mathbf{Z}=\mathbf{z})=y \mid \mathbf{Z} \neq \mathbf{z}, Y \neq y)}_{\text {(1) sufficient }} \\
&+\underbrace{\mathbb{E}_{p(\mathbf{Z}=\mathbf{z}, Y=y)} P(Y(\mathbf{Z} \neq \mathbf{z}) \neq y \mid \mathbf{Z}=\mathbf{z}, Y=y)}_{\text {(2) necessary }}
\end{aligned}
\end{equation}
\end{definition}
PNS is identifiable if and only if $\mathbf{Z}$ is the parent of $Y$ \citep{pearl2009causality}. To be more specific, sufficient (Eq.\ref{eq:pns}(1)) describes the capability of the causal representation to generate labels and necessary (Eq.\ref{eq:pns}(2)) indicates the probability of negative counterfactual label if we intervene on representation $\mathbf{z}$. Although it is theoretically well sound, the estimation of PNS is with some difficulty since  $P(Y (\mathbf{Z = z}) = y|\mathbf{Z \not= z}, Y\not=y)$ and $P(Y (\mathbf{Z \not= z}) \not= y|\mathbf{Z = z}, Y=y)$ are both counterfactual distributions, e.g., the sufficient term in Eq.\ref{eq:pns} means $P(Y=y)$ given the observation $\mathbf{Z\not=z}, Y\not=y$, if $\mathbf{Z}$ is intervened to be $\mathbf{z}$ ($do(\mathbf{Z = z})$). To enable Eq.\ref{eq:pns} to be tractable during learning, we consider an objective equivalent to Definition \ref{def:PNS} as:
\begin{equation}\label{eq:pns1}
\begin{split}
    \text{PNS} = & \mathbb{E}_{p(\mathbf{Z \not= z}, Y\not=y)}P(Y (\mathbf{Z = z}) = y|\mathbf{Z \not= z}, Y\not=y) \\
    &- \mathbb{E}_{p(\mathbf{Z = z}, Y=y)}P(Y (\mathbf{Z \not= z}) = y|\mathbf{Z = z}, Y=y)
\end{split}
\end{equation}
\begin{figure*}[t]
\centering
\setlength{\fboxrule}{0.pt}
\setlength{\fboxsep}{0.pt}
\fbox{
\includegraphics[width=0.5\linewidth]{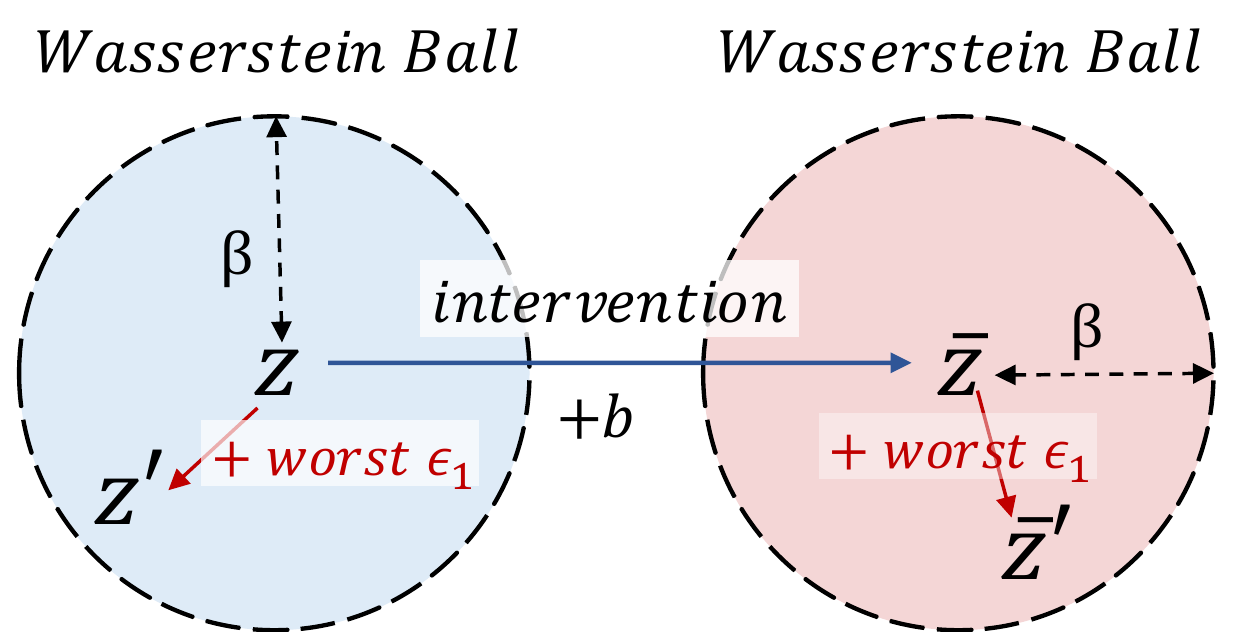}
}
\caption{{The intervention process, where $\mathbf{z}$ is intervened to be $\mathbf{\bar{z}}$. $\mathbf{z}$ and  $\mathbf{\bar{z}}$ are perturbed as $\mathbf{z'}$ and  $\mathbf{\bar{z}'}$ in Wasserstein ball.}}
\label{fig:cv}
\end{figure*}

As shown in Fig.\ref{fig:cv}, denote  $\mathbf{\bar{z}} $ the intervened $\mathbf{z}$. The intervention process is controlled by a bias term $b$ which is considered a hyperparameter where $\mathbf{\bar{z}} = \mathbf{{z}} + b$.
We use this term to specify $\mathbf{Z \not= z}$ by changing $\mathbf{Z = \bar{z}}$.
As illustrated in Preliminary, counterfactual distribution estimation normally follows a 3-step procedure in Pearl's framework, under which we propose a constrained method  for counterfactual estimation. Considering the causal generative process $\mathbf{Y} = f(\mathbf{pa_Y}, \bm{\epsilon}_1)$, the $\bm{\epsilon}_1$ is regarded as a random noise perturbing the $\mathbf{pa_y}$ inside a ball with finite diameter.  We treat the inference approach as the process of adversarial attack \citep{szegedy2013intriguing, ben2009robust, biggio2018wild} and define the 'Actions'-step in counterfactual estimation as
\begin{equation}\label{eq:action_step}
\begin{split}
\mathbf{z}' = \mathbf{z} + \bm{\epsilon}_1, \mathbf{z}'\in \mathcal{B}(\mathbf{z}, \beta)\\
\mathbf{\bar{z}}' = \bar{\mathbf{z}}+ \bm{\epsilon}_1, \mathbf{\bar{z}}'\in\mathcal{B}( \mathbf{\bar{z}}, \beta)
\end{split}
\end{equation}

where $\mathcal{B}(\mathbf{z}, \beta)$ is Wasserstein ball, in which the $p$-th Wasserstein distance \citep{panaretos2019statistical} $\mathrm{W}_{p}$ \footnote{$\mathrm{W}_{p}(\mu, \nu)=\left(\inf _{\gamma \in \Gamma(\mu, \nu)} \int_{\mathcal{Z} \times \mathcal{Z}} \Delta\left(\boldsymbol{z}, \boldsymbol{z}^{\prime}\right)^{p} d \gamma\left(z, z^{\prime}\right)\right)^{1 / p}$, $\Gamma(\mu, \nu$ is the collection of all probability measures
on $\mathcal{Z} \times \mathcal{Z}$}between $z$ and $\mathbf{z}'$ is smaller than $\beta$. {$\mathbf{z'}$ and $\mathbf{\bar{z}'}$ integrate both intervention and exogenous information.}
We further define an counterfactual vulnerability (CV) to measure PNS from the perspective of mutual information. The defined counterfactual vulnerability term quantifies the vulnerability of prediction model (classification) under input perturbations, which is formed later. It adjusts the parameterized $\mathbf{Z}$ to fit PNS Eq.\ref{eq:pns1}, which later can be estimated via maximum likelihood as following definition. 


\begin{definition} (Counterfactual Vulnerability) \label{def:cv_term}
Let $\mathbf{\bar{Z}}'$ denote intervened variables on $\mathbf{Z}=\phi(\mathbf{X})$, $\forall \mathbf{z}'\in \mathcal{B}(\mathbf{z}, \beta), \mathbf{\bar{z}}'\in\mathcal{B}( \mathbf{\bar{z}}, \beta)$, $D$ and $D'$ denote datasets sample from $p(\mathbf{z', y})$ and $p(\mathbf{\bar{z}', y})$, the  vulnerability of robust counterfactual estimation is defined as
\begin{equation}
\begin{split}
    &\min_{\mathbf{z}'\in \mathcal{B}(\mathbf{z}, \beta), \mathbf{\bar{z}}'\in \mathcal{B}(\mathbf{\bar{z}}, \beta)}CV_\mathcal{B} \\  
    = & \min_{\mathbf{z}'\in \mathcal{B}(\mathbf{z}, \beta)} \frac{1}{|D|}\sum_{\mathbf{z}', \mathbf{y}} \log P(Y=y|\mathbf{Z}' = \mathbf{z}')\\
    &-\min_{\mathbf{\bar{z}}'\in \mathcal{B}(\mathbf{\bar{z}}, \beta)}\frac{1}{|D'|}\sum_{\mathbf{\bar{z}}', \mathbf{y}} \log P(Y=y|\mathbf{\bar{Z}}' = \mathbf{\bar{z}}') \\
    =&\min_{\mathbf{z}'\in \mathcal{B}(\mathbf{z}, \beta)}H(Y|\mathbf{Z}') - \min_{\mathbf{\bar{z}}'\in \mathcal{B}(\mathbf{\bar{z}}, \beta)}H(Y|\mathbf{\bar{Z}}') +H(Y) \\
    &- H(Y)\\
    =& \min_{\mathbf{z}'\in \mathcal{B}(\mathbf{z}, \beta)}I(Y;\mathbf{Z}') - \min_{\mathbf{\bar{z}}'\in \mathcal{B}(\mathbf{\bar{z}}, \beta)}I(Y;\mathbf{\bar{Z}}') 
\end{split}
\end{equation}

\end{definition}
\begin{remark}
We can interpret the approach under the Pearl's 'Abduction-Action-Prediction' three-level framework. 'Abduction'-step: In Definition \ref{def:cv_term}, search for the worst $\bm{\epsilon}_1$ under minimum process of CV term, which capture the vulnerability in system. 'Action'-step: The Eq.\ref{eq:action_step} describes. 'Prediction'-step: we use the formulation of SCMs $y = f(\mathbf{z}, \bm{\epsilon}_1)$ to predict counterfactual results. 
\end{remark}
Combining the CV term with original objective $\delta(\phi)$, we get the final objective function optimized by minmax approach. Equivalently, we only need to optimize $I(\mathbf{Z'};Y)$ rather than $I(\mathbf{Z};Y)+I(\mathbf{Z'};Y)$ since if the worst case $I(\mathbf{Z}';Y)$ is satisfied, $I(\mathbf{Z};Y)$ is satisfied. The  robust optimization objective function is $L_\text{rb}{\phi}$, where
{\small{{\begin{equation}\label{eq:l_rb}
\begin{split}
    &\max_{\phi} \min_{\mathbf{z}'\in \mathcal{B}(\mathbf{z}, \beta), \mathbf{\bar{z}}'\in \mathcal{B}(\mathbf{\bar{z}}, \beta)} I(\phi(\mathbf{X});Y)-\lambda I(\phi(\mathbf{X);X}) + CV_\mathcal{B} \\
    \ge&\max_{\phi} \min_{\mathbf{z}'\in \mathcal{B}(\mathbf{z}, \beta), \mathbf{\bar{z}}'\in \mathcal{B}(\mathbf{\bar{z}}, \beta)} \underbrace{I(\mathbf{Z'};Y)-\lambda I(\mathbf{Z;X})}_{\text{(1) positive}} -\underbrace{I(\bar{\mathbf{Z}}';Y) }_{ \text{(2) negative}} \\
    =&L_\text{rb}(\phi)
\end{split}
\end{equation}}}}

Our model learning is accompanlished by the minmax procedure. Literally, the minimization procedure helps identifying the counterfactual interventional output under pertubations, or equalvalently, infer the exogenous variables from the perspective of adversarial learning. The maximization procedure is to identify a function that learns the representation most likely to satisfy the PNS condition. By optimizing the positive term, we aim at finding sufficient parent information, and optimizing the negative term is to find necessary parent information, so that the final optimization objective can extract minimal sufficient parents from observation data with high probability. 

\section{Sample Complexity}
In this section, we theoretically analyze the proposed algorithm by the probability approximately correct (PAC) framework. We start  from the perspective of information theory, and it in fact can be generalized to deep learning models.  We answer the question why causal representation containing cause information enhances the generalization ability. Different with traditional PAC learning theorem, we  analyze the risk by mutual information, which follows the framework of information bottleneck \citep{shamir2010learning}. We provide a finite sample bound of generalization ability. The bound measures the relationship between $I(\mathbf{Z}; Y)$ and its estimation $\hat{I}(\mathbf{Z}; Y)$. Here, we provide theoretical justification with following theory (proven in Appendix \ref{sup:proof}):


\begin{theorem}\label{thm:sample_complexity}
Let $\mathbf{Z}=\phi(\mathbf{X})$ where $\phi: \mathcal{X}\rightarrow\mathcal{Z}$ be a fixed arbitrary function, determined by a known conditional probability distribution $p(\mathbf{z|x})$. Let $m$ be sample size  and the joint distribution is  $p(\mathbf{X}, Y)$. For
any confidence parameter $0<\delta<1$, it holds with a probability of at least $1 - \delta$, that

1. General case ($\mathbf{Z = \phi(X)}$)
    {\small{\begin{equation}
    \begin{aligned}
        &|I(Y ; \mathbf{Z})-\hat{I}(Y ; \mathbf{Z})| \leq\\
        &\frac{\sqrt{C \log (|\mathcal{Y}| / \delta)}\left(|\mathcal{Y}|\sqrt{|\mathcal{Z}|} \log (m)+\frac{1}{2} \sqrt{|\mathcal{Z}|} \log (|\mathcal{Y}|)\right)+\frac{2}{e}|\mathcal{Y}|}{\sqrt{m}}
    \end{aligned}
    \end{equation}}}
    where $m \geq \frac{C}{4} \log (|\mathcal{Y}| / \delta)|\mathcal{Z}|\mathrm{e}^{2}$
    
2. Ideal case ($\mathbf{Z = \phi(X) = pa_Y}$)
    {\small{\begin{equation}
    \begin{split}
        &|I(Y ; \mathbf{Z})-\hat{I}(Y ; \mathbf{Z})|\\
        &\leq \frac{\sqrt{C \log (|\mathcal{Y}| / \delta)}\left(|\mathcal{Y}|\sqrt{\beta} \log (m)+\frac{1}{2} \sqrt{|\mathcal{Z}|} \log (|\mathcal{Y}|)\right)+\frac{2}{e}|\mathcal{Y}|}{\sqrt{m}}\nonumber
    \end{split}
    \end{equation}}}
    where $m \geq C \log (|\mathcal{Y}| / \delta)\beta\mathrm{e}^{2}$
\end{theorem}
\begin{remark}
The theorem provides a generalization bound under finite sample settings. It shows that when representation $\mathbf{Z}$ fully contains parent information $\mathbf{pa_Y}$, we achieve a sample complexity bound as $m \geq C \log (|\mathcal{Y}| / \delta)\beta\mathrm{e}^{2}$, where $\beta$ refers to Eq.\ref{eq:scms}. The minimum number of samples needed reduces from $|\mathcal{Z}|$ to $\beta$, which is a tighter bound since in most of cases we assume $|\mathcal{Z}|\gg\beta$. 
This shows that $\mathbf{z=pa_Y}$ gives the reduced sample complexity and tightened generalization bound. The theorem also serves as a general solution to causality prediction problems, supporting the claim that a better prediction is achieved with causal variables, compared to that with correlated variables.
\end{remark}
\section{Experiments}
In this section, we conduct extensive experiments to verify the effectiveness of our framework. In the following, we begin with the experiment setup, and then report and analyze the results.

\begin{table*}\label{tab:generalization}
\caption{Overall Results on Yahoo!R3-OOD, Yahoo!R3-i.i.d. and PCIC}
    \center
\small
\renewcommand\arraystretch{1.1}
\setlength{\tabcolsep}{3.3pt}
\begin{threeparttable}  
\scalebox{.95}{
    \begin{tabular}{c|c|cc|cc|cc|cc}
    \hline\hline
    Dataset& Method&\multicolumn{4}{c|}{p=$\infty$}&\multicolumn{4}{c}{p=2} \\ \hline
        & Metrics & AUC & ACC & advAUC & advACC & AUC & ACC & advAUC & advACC \\ \hline
        \multirow{7}{*}{Yahoo!R3-OOD}&base(robust) & 0.5 & 0.4508 & 0.5 & 0.4508 & 0.5 & 0.4545 & 0.5 & 0.4537 \\ 
        &base(standard) & 0.6198 & 0.6097 & 0.5212 & 0.5189 & 0.621 & 0.6099 & 0.5139 & 0.5188 \\
        &IB(standard) & 0.6181 & 0.6063 & 0.5333 & 0.5149 & 0.6184 & 0.6069 & 0.5431 & 0.5255 \\
        &r-CVAE(robust) & 0.6186 & 0.6235 & 0.5886 & 0.5912 & 0.6191 & 0.6241 & 0.5882 & 0.5907 \\
        &r-CVAE(standard) & 0.6253 & 0.6249 & 0.5855 & 0.5863 & 0.6233 & 0.6243 & 0.5865 & 0.5872 \\
        &CaRR(robust) & 0.6238 & \textbf{0.6284} & \textbf{0.5993} & \textbf{0.5999} & 0.6242 & \textbf{0.6307} & \textbf{0.6008} & \textbf{0.601} \\ 
        &CaRR(standard) & \textbf{0.629} & 0.6257 & 0.5966 & 0.5965 & \textbf{0.6276} & 0.6255 & 0.5917 & 0.5917 \\  \hline
        \multirow{7}{*}{Yahoo!R3-i.i.d.}&base(robust) & 0.5 & 0.6001 & 0.5 & 0.5997 & 0.5 & 0.6 & 0.5 & 0.6 \\ 
        &base(standard) & 0.7334 & 0.7483 & 0.6267 & 0.6251 & 0.7346 & 0.752 & 0.6260 & 0.6103 \\ 
        &IB(standard) & 0.7291 & 0.7513 & 0.6361 & 0.6721 & 0.7348 & 0.7521 & 0.6418 & 0.6775 \\ 
        &r-CVAE(robust) & 0.7341 & 0.7093 & 0.7180 & 0.7080 & 0.7376 & 0.7151 & 0.7194 & 0.7082 \\ 
        &r-CVAE(standard) & 0.7488 & 0.7515 & 0.7191 & 0.7072 & 0.7487 & 0.7529 & 0.7202 & 0.7099 \\
        &CaRR(robust) & {0.7378} & 0.7168 & \textbf{0.721} & \textbf{0.7107} & {0.7374} & {0.7158} & \textbf{0.7247} & \textbf{0.7159} \\ 
        &CaRR(standard) & \textbf{0.7497} & \textbf{0.7503} & 0.7191 & 0.7099 & \textbf{0.7493} & \textbf{0.7495} & 0.7188 & 0.7072 \\  
        \hline
        \multirow{7}{*}{PCIC}&base(robust) & 0.5534 & 0.5875 & 0.5388 & 0.6257 & 0.5605 & 0.6498 & 0.5264 & 0.6287 \\
        &base(standard) & 0.6177 & 0.6517 & 0.5231 & 0.589 & 0.6269 & 0.6615 & 0.519 & 0.5581 \\ 
        &IB(standard) & 0.6242 & 0.6532 & 0.5741 & 0.6199 & 0.6216 & 0.6537 & 0.5768 & 0.6233 \\ 
        &r-CVAE(robust) & 0.6363 & 0.6733 & 0.6088 & 0.6596 & 0.63 & 0.674 &  0.6187 & 0.6493 \\ 
        &r-CVAE(standard) & 0.6358 & 0.6779 & 0.6138 & 0.6601 & 0.6328 & 0.6725 & 0.5893 & 0.6429  \\ 
        &CaRR(robust) & {0.639} & 0.6761 & \textbf{0.6225} & { 0.6638} & {0.6363} & {0.6709} & \textbf{0.6332} & {0.6576} \\ 
        &CaRR(standard) & \textbf{0.6447} & \textbf{0.6817} & 0.6148 & \textbf{0.664} & \textbf{0.6416} & \textbf{0.6803} & 0.619 & \textbf{0.6625} \\ 
        \hline\hline
    \end{tabular}}   
\end{threeparttable}    
\label{tab:overall_yahoo}
\end{table*}

\subsection{Datasets}
{Our experiments are based on one synthetic and four real-world benchmarks. With the synthetic dataset, we evaluate our method in controlled manner under selected dataset. We follow the causal graph defined in Fig.\ref{fig:intro} (a) to build our synthetic simulator, on which we compare the representation learnt by our method with the ground truth under different $\beta$ degree. The detail of synthetic simulator is given in Appendix \ref{sup:exp_details}}.

We also evaluate our method on real-world benchmarks for the recommendation system.\textbf{Yahoo! R3}\footnote{https://webscope.sandbox.yahoo.com/catalog.php?datatype=r} is an online music recommendation dataset, which contains the user survey data and ratings for randomly selected songs. The dataset contains two parts: the uniform (OOD) set and the nonuniform (i.i.d.) set.
\textbf{Coat Shopping Dataset}\footnote{https://www.cs.cornell.edu/~schnabts/mnar/} is a commonly used dataset collected from web-shop rating on clothing. The self-selected ratings are the i.i.d. set and the uniformly selected ratings are the OOD set.
\textbf{PCIC}\footnote{https://competition.huaweicloud.com/} is a recently released dataset for debiased recommendation, where we are provided with the user preferences on uniformly exposed movies. {We also use \textbf{CelebA-anno} \footnote{http://mmlab.ie.cuhk.edu.hk/projects/CelebA.html} data as one of our benchmark dataset, which contains 8 selected manual labeled facial annotations on 202,599 human face images.}

\subsection{Experimental Setup}\label{sec:exp_steup}
\textbf{Implementation.}
We name our method as \textbf{CaRR}. All of objective functions are defined under information-theoretical formulation. We evaluate Eq.\ref{eq:l_rb} in two parts. The first positive part (Eq.\ref{eq:l_rb} (1)) is evaluated by the following parameterized objective \citep{alemi2016deep}:
\begin{equation}\label{eq:elbo}
\begin{split}
    &I(\mathbf{Z}'; Y)- \lambda I(\textbf{Z}; \mathbf{X}) \\
    \ge& \mathbb{E}_{D}[\mathbb{E}_{\mathbf{z'}\in\mathcal{B}(\mathbf{z}, \beta)}[\log p_{g}({y}|\mathbf{z'})]-\lambda\mathcal{D_\text{KL}}(q_{\bm{\phi}} (\mathbf{z}|\mathbf{x})||p_{\bm{\theta}}(\mathbf{z}))] 
\end{split}
\end{equation}
we use PGD attack \citep{madry2017towards} with $\infty$-norm and $2$-norm to get intervened $Z'$. We set $p_{\bm{\theta}}(\mathbf{z})$ as $\mathcal{N}(y, 1)$ to avoid trivial representations. For the negative term (Eq.\ref{eq:l_rb} (2)), let $\bar{\mathbf{z}} = \mathbf{z}+b$, and $b$ is hyperparameter denoting degree of bias. In our experiments, we set $b=0.8$ when $y=0$, $b=-0.8$ when $y=1$. Then we use 
negative cross entropy to approximate mutual information. More implementation details are shown in Appendix \ref{sup:exp_details}.
\begin{figure*}[htp]
\begin{center}
\centerline{\includegraphics[width=1\columnwidth]{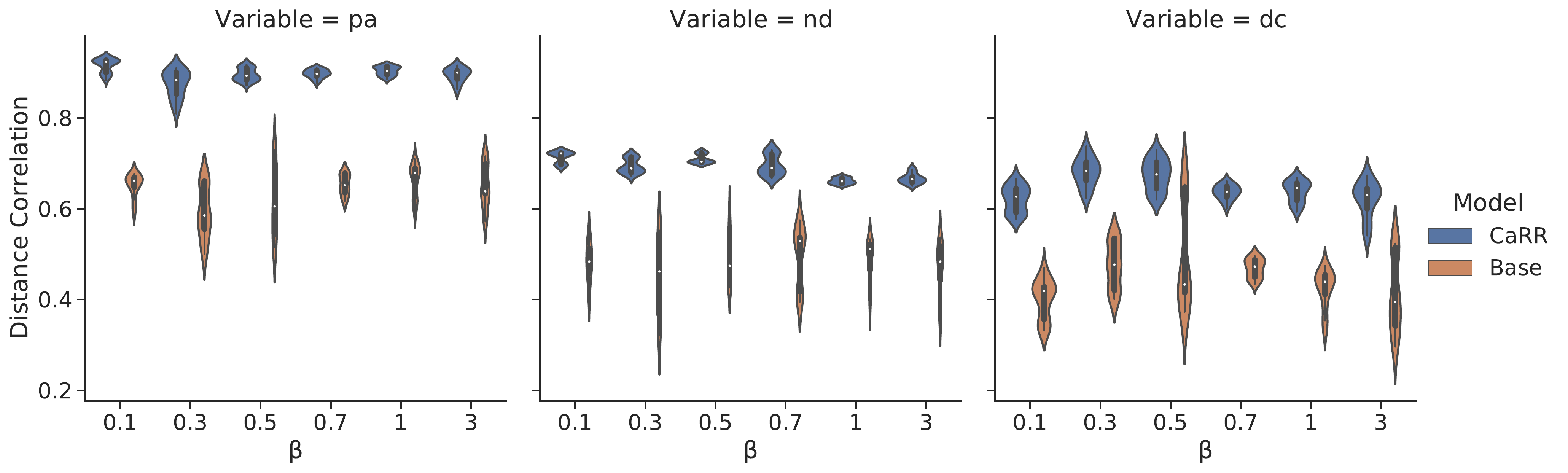}}
\caption{Representation learning results on synthetic dataset over different range of $\beta$, where $p=2$ under robust training.}
\label{fig:identify_result}
\end{center}
\vspace{-9mm}
\end{figure*}

\textbf{Compared Method}.  For all the compared methods, we use the same model architecture, with different training strategies. The model  consists of representation learning module $\mathbf{z}=\phi(\mathbf{x})$  and the downstream prediction module  $\mathbf{\hat{y}} = g(\mathbf{z})$, with each module implemented by neural networks. \textbf{Base} model has no additional constraints on representation, and the optimization is to minimize the cross-entropy between ${y}$ and learned ${\hat{y}}$. We  involve a recently proposed variational estimation with information bottleneck (\textbf{IB}) \citep{alemi2016deep}, extend the condition VAE (CVAE \citep{sohn2015learning}) by robust training process as \textbf{r-CVAE}, whose objective function is similar with CaRR but without a negative term (Eq.\ref{eq:l_rb} (2)). We conduct ablation studies by comparing our proposed method \textbf{CaRR} with the r-CVAE to evaluate the effectiveness of negative term.  We evaluate our method on two main aspects: (i) \textbf{Generalization} of the model under distribution shifts and (ii) \textbf{Robustness} under adversarial attack on representation space. 
For (i), we evaluate our method on OOD and i.i.d. setting on Yahoo! R3 and Coat. For (ii), the standard mode of adversarial attack  ($\beta=0$) means that we do not perturb original $\mathbf{z}$. In robust mode, we set $\beta=\{0.1, 0.2, 0.1, 0.3, 0.3\}$ for PCIC, Yahoo! R3, Coat, Synthetic and CelebA-anno respectively. 

\textbf{Metrics}. We use the common evaluation metrics AUC/ACC \citep{rendle2012bpr, gunawardana2009survey} on CTR prediction and their variants called adv-ACC/ adv-AUC \citep{madry2017towards} on advasarially perturbed evaluation dataset. {Moreover, we consider Distance Correlation metrics \cite{jones1995fitness} to evaluate the similarity between learned representation and parental information $\mathbf{pa_Y}$. }

\subsection{Overall Effectiveness}
Table \ref{tab:overall_yahoo} shows the overall results on Yahoo! R3 and PCIC. From Yahoo! R3 dataset, it contains both i.i.d. and OOD validation and test sets, we find that our method enjoys the better generalization. In Yahoo! R3 OOD, our method increases the performance by 1.9\%,  8.1\%, in terms of ACC and adv-ACC, compared with base method. The performance of r-CVAE is close to CaRR, since it is a modified version of our method, which only includes the positive term in Eq.\ref{eq:l_rb} but removes the negative term. The difference between performances of CaRR and r-CVAE shows the effectiveness of the negative term in the objective function of CaRR. In PCIC dataset, standard and robust modes of CaRR achieve the best AUC as 64.47\%, 63.9\% respectively, which validates the effectiveness of our idea. In the robust training mode, our method achieves the best performance in adversarial metrics. In PCIC dataset, our method reaches 62.25\%, which increases 8.37\% against base method on adv-AUC. Robust training of CaRR is also better than the standard training, winning with a margin around 1.42\%. The results show that the robust learning process with exogenous variables involved enhances the adversarial performance on perturbed samples. On the other hand, in standard training mode, CaRR achieves better adversarial performance than all baselines. The result supports that the causal representation our method learned is more robust.

\subsection{Representation Analysis}
{In this section, we study whether our method CaRR helps identifying the parental information from observational data. Fig.\ref{fig:identify_result} demonstrates the ability of the model on learning causal representations under different $\beta$ degree on synthetic dataset. The figure shows the distance correlation between the learned representation $\mathbf{z}$ and different parts of observational data, namely ($\mathbf{pa_Y, nd_Y, dc_Y}$). From Fig.\ref{fig:identify_result} (left), we find that our method learns a representation that is with the highest similarity, in comparison with base method under different values of $\beta$. It is an evidence that our method successfully identify the parental information from mixed observational data. The information from $\mathbf{nd_Y}$ and $\mathbf{dc_Y}$ are not considered as important as parental information from CaRR, and the  distance correlation metric corresponding to this part is slightly lower. We also find that the metric under CaRR gets lower variance, which shows the stable performance of CaRR. On the contrary, the distance correlation metric of base method is with high variance, which indicates the possible incapability of the base method on extracting the parental information from observations.}

\section{Conclusions}
In this paper, we deal with the problem of representation learning for robust deep models. We argue that when the observations contain mixed information of the labels, that is, the cause, effect and other unrelated variables, learning a representation largely preserving the cause information results in satisfactory generalizability of the models. By analyzing our hypothetical graphical model in information theory, we propose a causality-inspired representation learning method by regularized mutual information. Through counterfactual based model tuning, it achieves effective learning with guaranteed sample complexity reduction under certain assumptions. Extensive experiments on real data set show the effectiveness of our algorithm, verifying our claim of robust learning.
\newpage

\bibliography{citations.bib}
\bibliographystyle{iclr2019_conference}

\newpage
\onecolumn
\appendix

\section{Theoretical Proof} \label{sup:proof}
\subsection{Proof of Lemma \ref{lem:dpi}}

To start with, we firstly consider the base case (Fig.\ref{fig:intro} (a)). 

Let $I(\mathbf{X}; {Y})$ denote the mutual information between $\mathbf{X}$ and ${Y}$, and $H(\mathbf{X})$ denote the entropy of $\mathbf{X}$. From Data Processing Inequality (DPI), we get the the inequality of mutual information inside SCMs.

In confounded case shown in Fig.\ref{fig:intro} (b),  DPI does not apply since $\mathbf{dc_Y}$ is generated by both $\mathbf{pa_Y}$ and ${Y}$, $\mathbf{dc_Y}$. We thus decompose the mutual information in confounded case as
\begin{equation}
\begin{split}
\label{eq:mutual}
    &I(\mathbf{pa_Y; nd_Y, dc_Y})\\ =&I(\mathbf{pa_Y; nd_Y}, Y, \mathbf{dc_Y})-\underbrace{I(\mathbf{pa_Y; } Y|\mathbf{dc_Y, nd_Y})}_{\ge0}\\
    \le& I(\mathbf{pa_Y; nd_Y}, Y)+I(\mathbf{pa_Y; dc_Y}|Y, \mathbf{nd_Y}) \\
    =& I(\mathbf{pa_Y; nd_Y, }Y)+I(\mathbf{pa_Y;}\bm{\epsilon_3})=I(\mathbf{pa_Y; nd_Y, }Y)
\end{split}
\end{equation}

Compared with Eq.\ref{eq:dpi}, the right hand side of inequality in confounded case is with an additional mutual information term $I(\mathbf{pa_Y;}\bm{\epsilon_3})$. The $\mathbf{pa_Y}\perp{\bm{\epsilon_3}}$  term can be canceled out.
\subsection{Proof of Theorem \ref{thm:sufficient_statistics}}
The proof directly follows the following two lemmas. We denote  the set of probabilistic functions of $\mathbf{X}$ into an arbitrary target space as $\mathcal{F}(\mathbf{X})$, and as $\mathcal{S}({Y})$ the set of sufficient statistics for ${Y}$. Since $\mathbf{h}(\cdot)$ is an injective function, we have $I(Y; \mathbf{pa_y, nd_y, dc_Y}) = I(Y; \mathbf{X})$
\begin{lemma}
Let $\mathbf{Z}$ be a probabilistic function of $\mathbf{X}$. Then $\mathbf{Z}$ is a sufficient statistic for (${Y}$,$\mathbf{nd_Y}$) if and only if $I(Y, \mathbf{\mathbf{nd_Y}; pa_Y}) = I(Y,\mathbf{\mathbf{nd_Y}; \mathbf{Z}}) = \max_{\mathbf{Z'}\in \mathcal{F}(\mathbf{X})} I(Y, \mathbf{\mathbf{nd_Y}; \mathbf{Z'}})$
\end{lemma}
\begin{proof}
The proof of Lemma is an extension of  Lemma 12 in \citet{shamir2010learning}. The difference is that we focus on the joint variables of $Y\mathbf{, nd_Y}$ rather than ${Y}$. For every $\mathbf{Z'}$
which is a probabilistic function of $\mathbf{X}$, we have Markov Chain  $Y,\mathbf{ nd_Y} - \mathbf{X} - \mathbf{Z}'$, so we have $I(Y, \mathbf{nd_Y}; \mathbf{X}) \ge I(Y, \mathbf{nd_Y}; \mathbf{Z'})$. 
We also have Markov Chain  $Y, \mathbf{nd_Y} - \mathbf{Z} - \mathbf{X} $, so we have $I(Y, \mathbf{nd_Y}; \mathbf{X}) \le I(Y, \mathbf{nd_Y}; \mathbf{Z})$
Then assume $I(Y, \mathbf{nd_Y}; \mathbf{Z}) = I(Y, \mathbf{nd_Y}; \mathbf{X})$. Since we also have Markov Chain  $Y, \mathbf{nd_Y} - \mathbf{X} - \mathbf{Z}$, it follows that $Y, \mathbf{nd_Y}$ and $\mathbf{X}$ are conditionally independent given $\mathbf{Z}$, and hence $\mathbf{Z}$ is a sufficient statistic.

\end{proof}

\begin{lemma} 
Let $\mathbf{Z}$ be sufficient statistics of  ${Y}$. Then $\mathbf{Z}$ is minimal sufficient statistic for $\mathbf{Y}$ if and only if $I(\mathbf{nd_Y, dc_Y; pa_Y}) = I(\mathbf{nd_Y, dc_Y; Z}) = \max_{\mathbf{Z'}\in \mathcal{S}(Y, \mathbf{nd_Y})} I(\mathbf{nd_Y, dc_Y; Z'})$
\end{lemma}
\begin{proof}
Firstly, let $\mathbf{Z}$ be a minimal sufficient statistic, and  $\mathbf{Z}'$
be some sufficient statistic. Because there is a function $\mathbf{Z} = f(\mathbf{Z'})$, we have Markov Chain $(\mathbf{nd_Y, dc_Y}) - {Y} - \mathbf{Z'} - \mathbf{Z}$, and we  get $I(\mathbf{nd_Y, dc_Y; Z}) \le I(\mathbf{nd_Y, dc_Y;Z})$. Similarly, following the proof of Lemma 13 in \citet{shamir2010learning}, we get the above Lemma.
\end{proof}

\subsection{Proof of Lemma \ref{lem:ine_i}}
\begin{proof}

For the Lemma \ref{lem:ine_i} (1) is hold since
\begin{equation}
\begin{split}
    I(\mathbf{pa_Y;nd_Y, dc_Y})\le I(\mathbf{pa_Y;nd_Y, dc_Y, pa_Y})\le I(\mathbf{pa_Y;X}).
\end{split}
\end{equation}

For Lemma \ref{lem:ine_i} (2) in main text. 

Firstly, we introduce Kolmogorove Complexity $K(\mathbf{x})$, which describes the shortest description length of string $\mathbf{x}$. Previous works give technical results about Kolmogorove Complexity in causality \citep{janzing2010causal}. 
\begin{definition}
[Algorithmic Mutual Information]

Let $x, y$ be two strings. Then the algorithmic mutual information of $x, y$ is:

\begin{equation}
    I(\mathbf{x}:\mathbf{y}) :\overset{+}{=} K(\mathbf{y}) - K
    (\mathbf{y}|\mathbf{x}^*).
\end{equation}
The mutual information is the number of bits that can be saved in the description of $\mathbf{y}$ when the shortest description of $\mathbf{x}$ is already known.
\end{definition}

\begin{lemma}[Entropy and Kolmogorov Complexity \citep{vapnik1999overview}]\label{lem:entropy}
Let $\mathbf{x}=\mathbf{x}_1,\mathbf{x}_2\cdots, \mathbf{x}_n$ be a sting whose symbols $\mathbf{x}_j \in \mathcal{X}$ are drawn i.i.d. from a probability distribution $P(\mathbf{X})$ over the finite alphabet $\mathcal{{X}}$. Slightly overloading notation, set $P(\mathbf{x}) := P(\mathbf{x}_1)\cdots P(\mathbf{x}_n)$. Let $H(\cdot)$ denote the Shannon entropy of a probability distribution. Then there is a constant $c$ for $\forall n$ such that
\begin{equation}
    H(P(\mathbf{X}))\le \frac{1}{n}E(K(\mathbf{x}|n))\le H(P(\mathbf{X})) + \frac{|\mathcal{X}|\log n}{n} + \frac{c}{n}
\end{equation}
where $E(\cdot)$ is short hand for the expected value with respect to $P(\mathbf{X})$. Hence
\begin{equation}
\lim_{n\rightarrow \infty}\frac{1}{n}E(K(\mathbf{x})) = H(\mathbf{X})
\end{equation}
\end{lemma}
\begin{lemma} (Recursive Form \citep{janzing2010causal})
Given the strings $x_1,\cdots, x_n$ and a directed acyclic graph $G$. Then the Kolmogrove Complexity has the recursive form:
\begin{equation}
    \begin{split}
        K\left(\mathbf{x}_{1}, \ldots, \mathbf{x}_{n}\right) :\overset{+}{=} \sum_{j=1}^{n} K\left(\mathbf{x}_{j} \mid \mathbf{pa}_{j}^{*}\right)
    \end{split}
\end{equation}

\end{lemma}
\begin{equation}
    \begin{split}
        I(\mathbf{nd_y}, y:\mathbf{pa_y}) :\overset{+}{=}& K(\mathbf{nd_y}, y) - K(\mathbf{nd_y}, y|\mathbf{pa_y}^*)\\
        :\overset{+}{=}&K(\mathbf{nd_y}, y)-K(\mathbf{nd_y}|\mathbf{pa_y}^*) - K(y|\mathbf{pa_y}^*)\\
        :\overset{+}{=}&K(y)+K(\mathbf{d_y})-I(\mathbf{nd_y}: y) \\
        &-K(\mathbf{nd_y}|\mathbf{pa_y}^*) - K(y|\mathbf{pa_y}^*)\\
        :\overset{+}{=}&I(y:\mathbf{pa_y}) + \underbrace{I(\mathbf{nd_y}:\mathbf{pa_y}) -I(\mathbf{nd_y}:y)}_{\ge 0}\\
        :\overset{+}{\ge}& I(y:\mathbf{pa_y})
    \end{split}
\end{equation}

Lemma \ref{lem:entropy} already shows that 
\begin{equation}
    \lim_{n\rightarrow\infty} \frac{1}{n}E(I(\mathbf{nd_y}, y:\mathbf{pa_y}) = I(\mathbf{nd_y}, y;\mathbf{pa_y},
    \lim_{n\rightarrow\infty} \frac{1}{n}E(y:\mathbf{pa_y}) = I(y;\mathbf{pa_y})\\
\end{equation}
We can accomplish the proof of Lemma \ref{lem:ine_i}
\end{proof}
\subsection{Proof of Theorem \ref{thm:sample_complexity}}
The proof follows  \cite{shamir2010learning} Theorem 3. The sketch of proof contains two steps: (i) we decompose the original objective $|I(Y;\mathbf{Z})- \hat{I}(Y;\mathbf{Z})|$ into two parts. (ii) for each part, we deduce the deterministic finite sample bound by concentration
of measure arguments on L2 norms of random vector.
\begin{equation}
\begin{split}
    |I(Y;\mathbf{Z})- \hat{I}(Y;\mathbf{Z})|\le|H(Y|\mathbf{Z})-\hat{H}(Y|\mathbf{Z})|+|H(Y)-\hat{H}(Y)|
\end{split}
\end{equation}
Let $h(x)$ denote a continuous, monotonically increasing and concave function.
\begin{equation}
\begin{split}
    h({x})= \begin{cases}0 & x=0 \\ {x} \log (1 / {x}) & 0<{x} \leq 1 / \mathrm{e} \\ 1 / \mathrm{e} & x>1 / \mathrm{e}\end{cases}
\end{split}
\end{equation}
for the term $|H(Y|Z)-\hat{H}(Y|Z)|$
\begin{equation}\label{eq:separate}
\begin{split}
    |H(Y|\mathbf{Z})-\hat{H}(Y|\mathbf{Z})|&=\left|\sum_{\mathbf{z}}(p(\mathbf{z}) H(Y \mid \mathbf{z})-\hat{p}(\mathbf{z}) \hat{H}(Y \mid \mathbf{z}))\right| \\
& \leq\left|\sum_{\mathbf{z}} p(\mathbf{z})(H(Y \mid \mathbf{z})-\hat{H}(Y \mid \mathbf{z}))\right|+\left|\sum_{\mathbf{z}}(p(\mathbf{z})-\hat{p}(\mathbf{z})) \hat{H}(Y \mid \mathbf{z})\right|
\end{split}
\end{equation}
For the  first summand in this bound, we introduce variable $\bm{\epsilon}$ to help decompose $p(y|\mathbf{z})$, where $\bm{\epsilon}$ is independent with the parents $\mathbf{pa_y}$ (i.e. $\bm{\epsilon}\perp\mathbf{pa_y}$)
\begin{equation}
\begin{split}
    &\left|\sum_{\mathbf{z}} p(\mathbf{z})(H(Y \mid \mathbf{z})-\hat{H}(Y \mid \mathbf{z}))\right|\\
    & \leq\left|\sum_{\mathbf{z}} p(\mathbf{z}) \sum_{{y}}(\hat{p}({y} \mid \mathbf{z}) \log (\hat{p}({y} \mid \mathbf{z}))-p(y \mid \mathbf{z}) \log (p(y \mid \mathbf{z})))\right| \\
& \leq \sum_{\mathbf{z}} p(\mathbf{z}) \sum_{y} h(|\hat{p}(y \mid \mathbf{z})-p(y \mid \mathbf{z})|) \\
&=\sum_{\mathbf{z}} p(\mathbf{z}) \sum_{y} h\left(\left|\sum_{\bm{\epsilon}} p(\bm{\epsilon} \mid \mathbf{z})(\hat{p}(y \mid \mathbf{z}, \bm{\epsilon})-p(y \mid \mathbf{z}, \bm{\epsilon}))\right|\right) \\
&=\sum_{\mathbf{z}} p(\mathbf{z}) \sum_{y} h(\|\hat{\mathbf{p}}(y \mid \mathbf{z}, \bm{\epsilon})-\mathbf{p}(y \mid \mathbf{z}, \bm{\epsilon})\| \sqrt{V(\mathbf{p}(\bm{\epsilon} \mid \mathbf{z}))})\\
\end{split}
\end{equation}
where $\frac{1}{m}V(x)$ denote the variance of vector $x$. For the second summand in Eq.\ref{eq:separate}.
\begin{equation}
\begin{split}
    \left|\sum_{\mathbf{z}}(p(\mathbf{z})-\hat{p}(z)) \hat{H}(Y \mid z)\right| \leq\|\mathbf{p}(\mathbf{z})-\hat{\mathbf{p}}(\mathbf{z})\| \cdot \sqrt{V(\hat{\mathbf{H}}(Y \mid \mathbf{z}))}
\end{split}
\end{equation}
For the summand $|H(Y)-\hat{H}(Y)|$:
\begin{equation}
\begin{split}
|H(Y)-\hat{H}(Y)| &=\left|\sum_{y} p(y) \log (p(y))-\hat{p}(y) \log (\hat{p}(y))\right| \\
& \leq \sum_{y} h(|p(y)-\hat{p}(y)|) \\
&=\sum_{y} h\left(\left|\sum_{\mathbf{z}} \sum_{\bm{\epsilon}}p(\bm{\epsilon} \mid \mathbf{z})(p(\mathbf{z})p( y|\bm{\epsilon})-\hat{p}(\mathbf{z})p( y|\bm{\epsilon}))\right|\right) \\
& \leq \sum_{y} h(\|\mathbf{p}(\mathbf{z})p(y|\bm{\epsilon})-\hat{\mathbf{p}}(\mathbf{z})p(y|\bm{\epsilon})\| \sqrt{V(\mathbf{p}(\bm{\epsilon} \mid \mathbf{z}))})
\end{split}
\end{equation}

Combining above bounds:
\begin{equation}
\begin{split}
    |I(Y;\mathbf{Z})- \hat{I}(Y;\mathbf{Z})|\le&\sum_{y} h(\|\mathbf{p}(\mathbf{z}, y|\bm{\epsilon})-\hat{\mathbf{p}}(\mathbf{z}, y|\bm{\epsilon})\|\sqrt{V(\mathbf{p}(\bm{\epsilon} \mid \mathbf{z}))}) \\
    &+\sum_{\mathbf{z}} p(\mathbf{z}) \sum_{y} h(\|\hat{\mathbf{p}}(y \mid \mathbf{z}, \epsilon)-\mathbf{p}(y \mid \mathbf{z}, \bm{\bm{\epsilon}})\| \sqrt{V(\mathbf{p}(\bm{\bm{\epsilon}} \mid z))})\\
    &+\|\mathbf{p}(\mathbf{z})-\hat{\mathbf{p}}(\mathbf{z})\| \cdot \sqrt{V(\hat{\mathbf{H}}(Y \mid \mathbf{z}))}
\end{split}
\end{equation}
Let $\bm{\rho}$ be a distribution vector of arbitrary cardinality, and let $\hat{\bm{\rho}}$ be an empirical estimation of $\bm{\rho}$ based on a sample of size $m$. Then the error $\|\bm{\rho}-\hat{\bm{\rho}}\|$ will be bounded  with a probability of at least $1-\delta$
\begin{equation}\label{eq:delta_}
\begin{split}
\|\bm{\rho}-\hat{\bm{\rho}}\| \leq \frac{2+\sqrt{2 \log (1 / \delta)}}{\sqrt{m}}
\end{split}
\end{equation}
Following the proof of Theorem 3 in \cite{shamir2010learning}, to make sure the bounds hold over $|\mathcal{Y}| + 2|$  quantities, we replace $\delta$ in Eq.\ref{eq:delta_} by $\delta/(|\mathcal{Y}| + 2|$, than substitute $\|\mathbf{p}(\mathbf{z}, y|\bm{\epsilon})-\hat{\mathbf{p}}(\mathbf{z}, y|\bm{\epsilon})\|$ $\|\hat{\mathbf{p}}(y \mid \mathbf{z}, \bm{\epsilon})-\mathbf{p}(y \mid \mathbf{z}, \bm{\epsilon})\|, \|\mathbf{p}(\mathbf{z})-\hat{\mathbf{p}}(\mathbf{z})\|, $ by Eq.\ref{eq:delta_}.
\begin{equation}\label{eq:i_y_z}
\begin{split}
    |I(Y ; \mathbf{Z})-\hat{I}(Y ; \mathbf{Z})| \leq &(2+\sqrt{2 \log ((|\mathcal{Y}|+2) / \delta)}) \sqrt{\frac{V(\hat{\mathbf{H}}(Y \mid \mathbf{z}))}{m}} \\
    &+2 |\mathcal{Y}| h\left(2+\sqrt{2 \log ((|\mathcal{Y}|+2) / \delta)} \sqrt{\frac{V(\mathbf{p}(\bm{\epsilon} \mid\mathbf{z}))}{m}}\right)\\
\end{split}
\end{equation}
There exist a constant $C$, where $2+\sqrt{2 \log ((|\mathcal{Y}|+2) / \delta)}\le\sqrt{C\log ((|\mathcal{Y}|) / \delta)}$. 
From the fact that variance of any random variable bounded in [0, 1] is at most 1/4, we analyze the bound under two different cases:

\textbf{In general case} ($\mathbf{z}=\phi(\mathbf{x})$), 
\begin{equation}
\begin{split}
    V(\mathbf{p}(\bm{\epsilon} \mid \mathbf{z}))\le \frac{|\mathcal{Z}|}{4}
\end{split}
\end{equation}
let $m$ denote the number of sample, we get a lower bound of $m$, which is also known as sample complexity.
\begin{equation}
\begin{split}
    m \geq \frac{C}{4} \log (|\mathcal{Y}| / \delta)|\mathcal{Z}| \mathrm{e}^{2}
\end{split}
\end{equation}

\textbf{In ideal case}( $z=pa_y$) $z\perp\epsilon$:
\begin{equation}
\begin{split}
    V(\mathbf{p}(\bm{\epsilon} \mid \mathbf{z}))\le \beta
\end{split}
\end{equation}
\begin{equation}
\begin{split}
    m \geq \frac{C}{4} \log (|\mathcal{Y}| / \delta)|\beta| \mathrm{e}^{2}
\end{split}
\end{equation}

\begin{equation}
    \sqrt{\frac{C \log (|\mathcal{Y}| / \delta) V(\mathbf{p}(\epsilon \mid \mathbf{z}))}{m}} \leq \sqrt{\frac{C \log (|\mathcal{Y}| / \delta)|\mathcal{Z}|}{4m}} \leq 1 / \mathrm{e}
\end{equation}
Then, from the fact that (\cite{shamir2010learning}):
\begin{equation}
\begin{split}
h\left(\sqrt{\frac{\nu}{m}}\right) &=\left(\sqrt{\frac{\nu}{m}} \log \left(\sqrt{\frac{m}{v}}\right)\right) \\
& \leq \frac{\sqrt{v} \log (\sqrt{m})+1 / \mathrm{e}}{\sqrt{m}},
\end{split}
\end{equation}
We can get the upper bound of second summand in Eq.\ref{eq:i_y_z} as follows
\begin{equation}\label{eq:yh}
\begin{split}
    &\sum_y h\left(\sqrt{C\log ((|\mathcal{Y}|) / \delta)} \sqrt{\frac{V(\mathbf{p}(\epsilon \mid z))}{m}}\right)\\
    \le& \frac{\sqrt{C \log (|\mathcal{Y}| / \delta)} \log (m)\left(|\mathcal{Y}|\sqrt{V(\mathbf{p}(\bm{\epsilon} \mid \mathbf{z}))}\right)+\frac{2}{\mathrm{e}}|\mathcal{Y}|}{2 \sqrt{m}}
\end{split}
\end{equation}
\textbf{In general case}:
\begin{equation}
\begin{split}
    Eq.\ref{eq:yh} \le \frac{\sqrt{C \log (|\mathcal{Y}| / \delta)} \log (m)\left(|\mathcal{Y}|\sqrt{\mathcal{|Z|}}\right)+\frac{2}{\mathrm{e}}|\mathcal{Y}|}{2 \sqrt{m}}
\end{split}
\end{equation}

\textbf{In ideal case}:
\begin{equation}
\begin{split}
    Eq.\ref{eq:yh} \le \frac{\sqrt{C \log (|\mathcal{Y}| / \delta)} \log (m)\left(|\mathcal{Y}|\sqrt{\beta}\right)+\frac{2}{\mathrm{e}}|\mathcal{Y}|}{2 \sqrt{m}}
\end{split}
\end{equation}
For the first summand in Eq.\ref{eq:i_y_z}, we follow the fact (\cite{shamir2010learning} Theorem 3) that:
\begin{equation}
\begin{split}
    V(\mathbf{H}(Y \mid \mathbf{z})) \leq \frac{|Z| \log ^{2}(|\mathcal{Y}|)}{4}
\end{split}
\end{equation}
Finally we accomplish the proof of Theorem \ref{thm:sample_complexity}.

\section{Experimental Details} \label{sup:exp_details}
\subsection{Synthetic Datasets}
The synthetic data is generated following the general causal graph Fig.\ref{fig:intro}. We build the simulator  using nonlinear functions refering to \cite{DBLP:conf/kdd/ZouKCC019, yang2021top}. We simulate 500 data for each settings. Let $\kappa_1(\cdot)$ and $\kappa_2(\cdot)$ as piecewise functions, and $\kappa_1(x) = x - 0.5$ if $x>0$, otherwise $\kappa_1(x) = 0$, $\kappa_2(x) = x$ if $x>0$, otherwise $\kappa_2(x) = 0$ and $\kappa_3(x) = x + 0.5$ if $x<0$, otherwise $\kappa_3(x) = 0$. . For the fair evaluation, we set the same dimension for $\mathbf{pa_Y, nd_Y, dc_Y}$ that $d_1=d_2=d_3=5$. The nonlinear systems are:
\begin{equation}
\begin{split}
    &\mathbf{pa_Y} \sim U(-1, 1), \\
    &\bm{\epsilon_1} = \bm{\epsilon_2} = \bm{\epsilon_3}\sim \mathcal{N}(0.3, \beta I)\\
    &\mathbf{nd}_{1} =  \bm{a}^T\kappa_1(\kappa_2([\mathbf{pa_Y}, \bm{\epsilon_2} ]))+{q},\mathbf{nd}_{2} =  \bm{a}^T\kappa_3(\kappa_2([-\mathbf{pa_Y},-\bm{\epsilon_2}]))+{q}, \mathbf{nd_Y} = \sigma(\mathbf{nd}_{1} + \mathbf{nd}_{1}\cdot \mathbf{nd}_{2})\\
    &\mathbf{y}_{1} =  \bm{a}^T\kappa_1(\kappa_2([\mathbf{pa_Y}, \bm{\epsilon_1} ]))+{q},\mathbf{y}_{2} =  \bm{a}^T\kappa_3(\kappa_2([-\mathbf{pa_Y},-\bm{\epsilon_1}]))+{q}, \mathbf{nd_Y} = \mathbb{I}(\sigma(\mathbf{y}_{1} + \mathbf{y}_{1}\cdot \mathbf{y}_{2}))\\
    &\mathbf{dc}_{1} =  \bm{a}^T\kappa_1(\kappa_2([y, \bm{\epsilon_3} ]))+{q},\mathbf{dc}_{2} =  \bm{a}^T\kappa_3(\kappa_2([-y,-\bm{\epsilon_3}]))+{q}, \mathbf{nd_Y} = \sigma(\mathbf{dc}_{1} + \mathbf{dc}_{1}\cdot \mathbf{dc}_{2})\\
    &\mathbf{X} = [\mathbf{pa_Y, nd_Y, dc_Y}]
\end{split}
\end{equation}
where $q=0.3$, $\mathbb{I}(x)$ is an indicator function, which is 1 if $x>0$, and 0 otherwise. From synthetic data, we analyze whether CaRR have the ability of identifying the $\mathbf{pa_Y}$ from mixed observational $\mathbf{X}$.

\subsection{Real-world Datasets}
\textbf{Yahoo! R3} The nonuniform (OOD) set contains samples of users deliberately selected, and rates the songs by preference, which can be considered as a stochastic logging policy. For the uniform (i.i.d.) set, users were asked to rate 10 songs randomly selected by the system. The dataset contains 14,877 users and 1,000 items. The density degree is 0.812\%, which means that the dataset only records 0.812\% of rating pairs.

\textbf{CPC} The dataset contains 85000 samples for training and 15000 samples for validation and test. The recommended item list will be exposed to query of the system by nonuniform recommendation policy. The data includes 29 dimensions of matching features, which includes query and item features. 

\textbf{Coat} The dataset is collected by an online web-shop interface. In training dataset, users were asked to rate 24 coats selected by themselves from 300 item sets. In test dataset, it collects the userrates on 16 random items from 300 item sets. Just as Yahoo! R3,the training dataset is a non-uniform dataset and the test dataset is uniform dataset. The dataset provides side information of both users and item sets. The feature dimension of user/item pair is 14/33.

\textbf{PCIC} The dataset  is
collected from a survey by questionnaires about the rate and reason why the audience like or dislike the movie. Movie features are collected form movie-review pages. The training data is a biased dataset consisting of 1000 users asked to rate the movies they care from 1720 movies. The validation and test set is the user preference on uniformly exposed movies. The density degree is set to be 0.241\%. 

For evaluation, Yahoo! R3 and Coat dataset both have two validation (include test) datasets. The i.i.d. set is 1/3 data from nonuniform logging policy, and OOD set consists of the data generated under a uniform policy. For PCIC dataset, we train our method on non-uniform datasets and perform evaluations on uniform dataset. 

\textbf{CelebA-anno} The dataset contains more than 200K celebrity images, each with 40 attribute annotations. Following the previous work \cite{kocaoglu2017causalgan}, we select 9 attribute annotations, which include Young, Male, Eyeglasses, Bald, Mustache, Smiling, Wearing Lipstick, Mouth Open. Our task is to predict Smiling. $\mathbf{pa_Y}$ including \{Young, Male\}, $\mathbf{nd_Y}$ including \{Eyeglasses, Bald, Mustache, Wearing Lipstick\} and $\mathbf{cd_Y}$ including \{Mouth Open\}. From this dataset, we evaluate the ability of distinguishing $\mathbf{pa_Y}$ from $\mathbf{X}$.

\subsection{Model Architecture and Implementation Details}
The hyper-parameters are determined by grid search.
Specifically, the learning rate and batch size are tuned in the ranges of {{$[10^{-1}, 10^{-2}, 10^{-3}, 10^{-4}]$} and $[64,128,256,512,1024]$}, respectively.
The weighting parameter $\lambda$ is tuned in $[0.001]$. Perturbation degrees are set to be $\beta=\{0.1, 0.2, 0.1, 0.3\}$ for Coat, Yahoo!R3, PCIC and CPC separately. The representation dimension is empirically set as 64. All the experiments are conducted based on a server with a 16-core CPU, 128g memories and an RTX 5000 GPU. The deep model architecture is shown as follows:

(1)Representation learning method $\phi(\mathbf{x})$:
If dataset is Yahoo!R3 or PCIC, in which only user id and item id are the input, we firstly use an embedding layer. The representation function architecture is:
\begin{itemize}
    \item Concat(Embedding(user id, 32), Embedding(item id, 32))
    \item Linear(64, 64), ELU()
    \item Linear(64, representation dim), ELU()
\end{itemize}
Then for the dataset Coat and CPC, the feature dimension is 29 and 47 separately. It do not use embedding layer at first. The representation function architecture is.
\begin{itemize}
    \item Linear(64, 64), ELU()
    \item Linear(64, representation dim), ELU()
\end{itemize}
(2)Downstream Prediction Model $g(\mathbf{z})$:
\begin{itemize}
\item Linear(representation dim, 64), ELU()
\item Linear(64, 2)
\end{itemize}
\begin{figure*}[h]
\begin{center}
\centerline{\includegraphics[width=0.6\columnwidth]{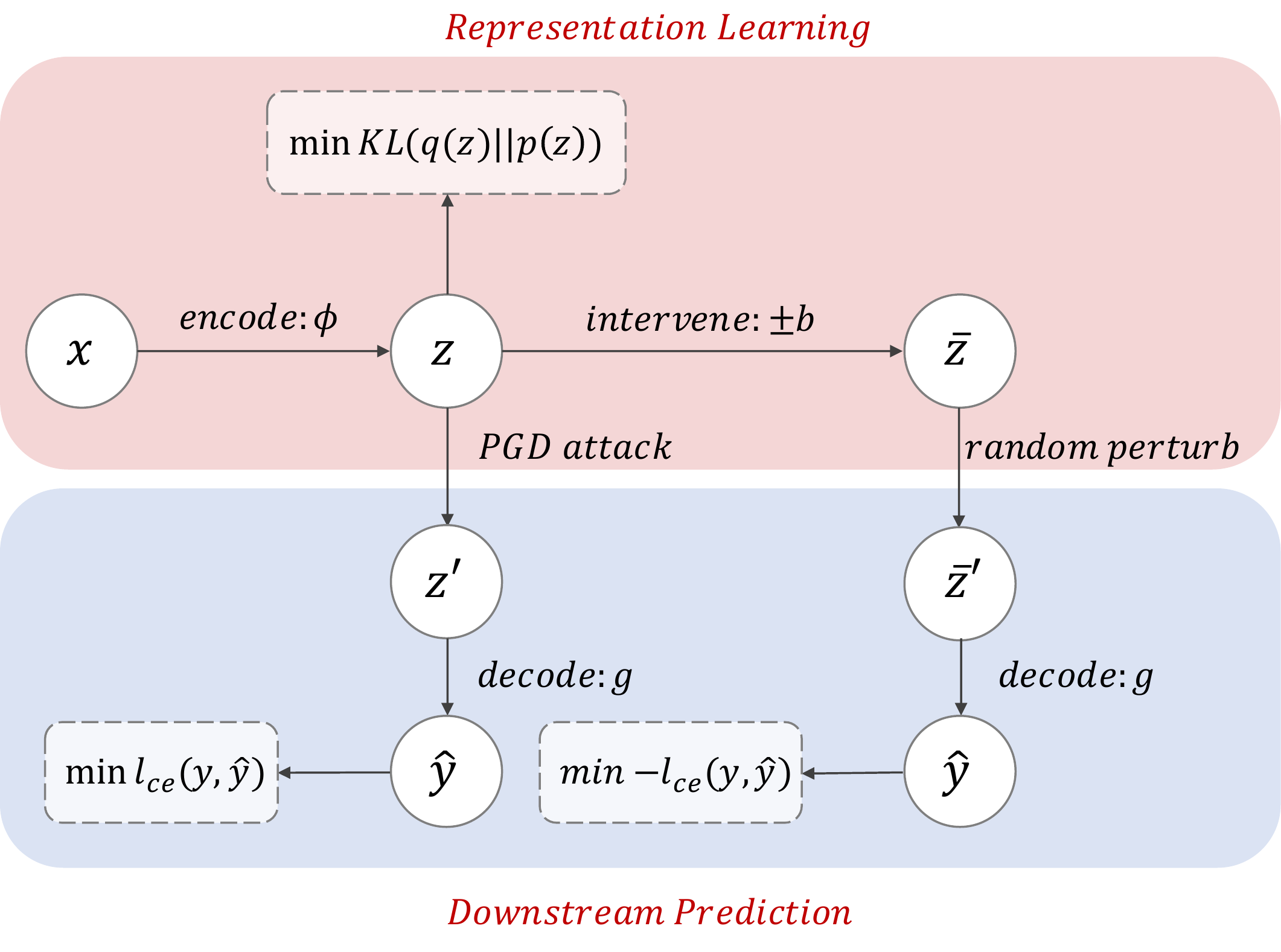}}
\caption{The figure demonstrates the model architecture of CaRR}
\label{fig:archi}
\end{center}
\end{figure*}

The figure shows the model architecture. The model consists of two parts, the representation learning part and downstream prediction part. As illustrated in Section \ref{sec:exp_steup}, our final objective is:
\begin{equation}\label{eq:final_ob}
 \mathbb{E}_{D}[\mathbb{E}_{\mathbf{z'}\in\mathcal{B}(\mathbf{z}, \beta)}[\log p_{g}(\mathbf{y}|\mathbf{z'})]-\lambda\mathcal{D_\text{KL}}(q_{\bm{\phi}} (\mathbf{z}|\mathbf{x})||p_{\bm{\theta}}(\mathbf{z}))-\mathbb{E}_{\mathbf{\hat{z}'}\in\mathcal{B}(\mathbf{\hat{z}}, \beta)}[\log p_{g}(\mathbf{y}|\mathbf{\hat{z}'})]]
\end{equation}
For the representation learning part, we firstly use encode function $\phi(\cdot)$ to get representation $\mathbf{z}$ and get the intervened $\mathbf{\hat{z}}$. Then we perturb the learned $\mathbf{z}$ by PGD attack procedure and perturb the $\mathbf{\hat{z}}$ by random perturbation to find the worst case correspnding to the worst downstream loss. Finally we put $\mathbf{z}'$ and $\mathbf{\hat{z}}'$ into the downstream prediction model $g(\cdot)$ to calculate $y$. The likelihood in Eq.\ref{eq:final_ob} is estimated by cross entropy loss.
Note that the perturbation approach would block the gradient propagation between representation learning process and downstream prediction by some implementation ways. Thus we use the conditional Gaussian prior $p_\theta(\mathbf{z}) = \mathcal{N}(y\mathbf{1}, \mathbf{I})$ rather than standard Gaussian distribution $p_\theta(\mathbf{z}) = \mathcal{N}(\mathbf{0}, \mathbf{I})$ to calculate KL term. If gradient propagation is blocked, by using conditional prior, the learning process of representation $\mathbf{z}$ and exogenous $\bm{\epsilon}$ embedded in $\mathbf{z}'$ will not be influenced. The form of conditional Gaussian prior is more general $p_\theta(\mathbf{z}) = \mathcal{N}(\zeta(y), \mathbf{I})$, where $\zeta(\cdot)$ could be any non trivial function like linear function even neural network.

\section{Additional Results} \label{sup:exp_results}
Due to the page limit in main text, we present the additional test results and analysis in this section. Table \ref{tab:overall_yahoo} and \ref{tab:coat_res} shows overall experimental results on Coat, The table contains both i.i.d. and OOD setting. Based on which we find that in most cases, our method  achieves better performance in terms of AUC and ACC, compared to base methods.  The results in \ref{tab:overall_yahoo} show that the robust learning process with exogenous variables involved  enhances the adversarial performance on perturbed samples. On the other hand, in standard training mode, CaRR achieves better adversarial performance than baselines including base method and IB. Although the robust training deteriorates the performance of on normal dataset, it will help to identify the causal representation, which benefits downstream prediction under adversarial attack. For instance, we find that standard training of CaRR on PCIC has an AUC of 64.47\%, which is better than the performance under robust training (63.9\%). But contrary conclusions are drawn on adversarial performance. The result supports that causal representation we learned is more robust. The performance of base method in robust training mode is worst in most of cases, indicating that robust training process will largely influence the learning of the model and ruin the prediction model. Although the robust training deteriorates the performance of on normal dataset, it will help to identify the causal representation, which benefits downstream prediction under adversarial attack.

Fig.\ref{fig:epsilon_result} demonstrates how robust training  degree ($\beta=\{0.1, 0.3, 0.5, 0.7, 1.0\}$) influences the downstream prediction under adversarial settings. We conduct the experiments on the attacked real-world dataset by PGD attacker. From Fig.\ref{fig:epsilon_result}, we find that our method is better than base method, because the base model's ability on standard prediction is broken by adversarial training. When $\beta$ is small, our method behaves closely to the r-CVAE in all the datasets. When $\beta$ gets larger, the difference between performance of CaRR and that of r-CVAE continuously enlarges in Yahoo!R3. In PCIC, the gap becomes the largest among all when $\beta=0.5$, and narrows down to 0 when $\beta=0.7$.  This is because in our framework, we explicitly deploy a model to achieve more robust representations, while others fail. 

Fig. \ref{fig:6} and \ref{fig:7} compare the distance correlation metric given by the training under standard and robust mode. It shows that our method performs consistently better compared with base methods in both modes, with a higher distance correlation, under smaller variance. The gap is obvious especially in the learning of parental information, which is the main focus of our approaches.

Fig. \ref{fig:8} and \ref{fig:9} record the results along optimization process and until convergence, under different settings of the pertubation degree $\beta$, considering the dataset CelebA-anno. The annotation smile is used as the label to be prediced, and other features are the source data.  It shows that when the optimization process is not finished, both approaches have similar performance, with unstability evidenced by large variance of the DC metric. However, our method outperforms the baseline when the optimization converges, owning a higher DC with smaller variations.  The results also show that $\beta$ is an important factor for training the model. Larger $\beta$ often leads to higher variance of the training of the model.

\begin{table*}[h]
    \center
\caption{Overall Results on Coat dataset.}
\renewcommand\arraystretch{1.1}
\setlength{\tabcolsep}{3.3pt}
\begin{threeparttable}  
\scalebox{.95}{
    \begin{tabular}{c|c|cc|cc|cc|cc}
    \hline\hline
    Dataset& Method&\multicolumn{4}{c|}{p=$\infty$}&\multicolumn{4}{c}{p=2} \\ \hline
        & Metrics & AUC & ACC & advAUC & advACC & AUC & ACC & advAUC & advACC \\ \hline
        \multirow{7}{*}{Coat-OOD}&base(robust) & 0.5586 & 0.5569 & 0.5479 & 0.5451 & 0.5593 & 0.556 & 0.5441 & 0.5412 \\
        &base(standard) & 0.5659 & 0.5724 & 0.3874 & 0.4024 & 0.5642 & 0.5687 & 0.3128 & 0.3317 \\ 
        &IB(standard) & 0.5659 & 0.5681 & 0.4701 & 0.4796 & 0.5659 & 0.5713 & 0.5442 & 0.5495 \\ 
        &r-CVAE(robust) & 0.5629 & 0.5586 & 0.559 & 0.5544 & 0.5634 & 0.5591 & 0.5572 & 0.5522 \\ 
        &r-CVAE(standard) & 0.5656 & 0.5643 & 0.5527 & 0.5478 & 0.5671 & 0.5649 & 0.5586 & 0.554 \\ 
        &CaRR(robust) & \textbf{0.5707} & {0.5681} & \textbf{0.5653} & \textbf{0.5659} & {0.5705} & {0.5675} & \textbf{0.5674} & \textbf{0.565} \\ 
        &CaRR(standard) & 0.5705 & \textbf{0.5718} & 0.5643 & 0.5659 & \textbf{0.5725} & \textbf{0.5732} & 0.5608 & 0.5601 \\ 
        \hline
        \multirow{7}{*}{Coat-i.i.d.}&base(robust) & 0.7156 & 0.7232 & 0.7034 & 0.7107 & 0.7195 & 0.7261 & 0.7001 & 0.7057 \\ 
        &base(standard) & 0.7191 & 0.7217 & 0.4911 & 0.487 & 0.7235 & 0.7255 & 0.3642 & 0.3515 \\ 
        &IB(standard) & 0.7162 & 0.72 & 0.6023 & 0.6017 & 0.7182 & 0.7222 & 0.694 & 0.696 \\ 
        &r-CVAE(robust) & 0.7147 & 0.7222 & 0.7105 & 0.7181 & 0.7087 & 0.7169 & 0.7058 & 0.7141 \\ 
        &r-CVAE(standard) & 0.7106 & 0.7184 & 0.7029 & 0.7106 & 0.7129 & 0.7206 & 0.7023 & 0.7059 \\ 
        &CaRR(robust) & {0.7276} & 0.7339 & \textbf{0.7208} & \textbf{0.727} & \textbf{0.7265} & \textbf{0.7331} & \textbf{0.7196} & \textbf{0.7261} \\ 
        &CaRR(standard) & \textbf{0.7283} & \textbf{0.7355} & 0.7125 & 0.7196 & 0.7248 & 0.7305 & 0.7069 & 0.7125 \\ 
        
        \hline\hline
    \end{tabular}
}   
\end{threeparttable}    
\label{tab:coat_res}
\end{table*}

\begin{figure*}[h]
\begin{center}
\centerline{\includegraphics[width=1\columnwidth]{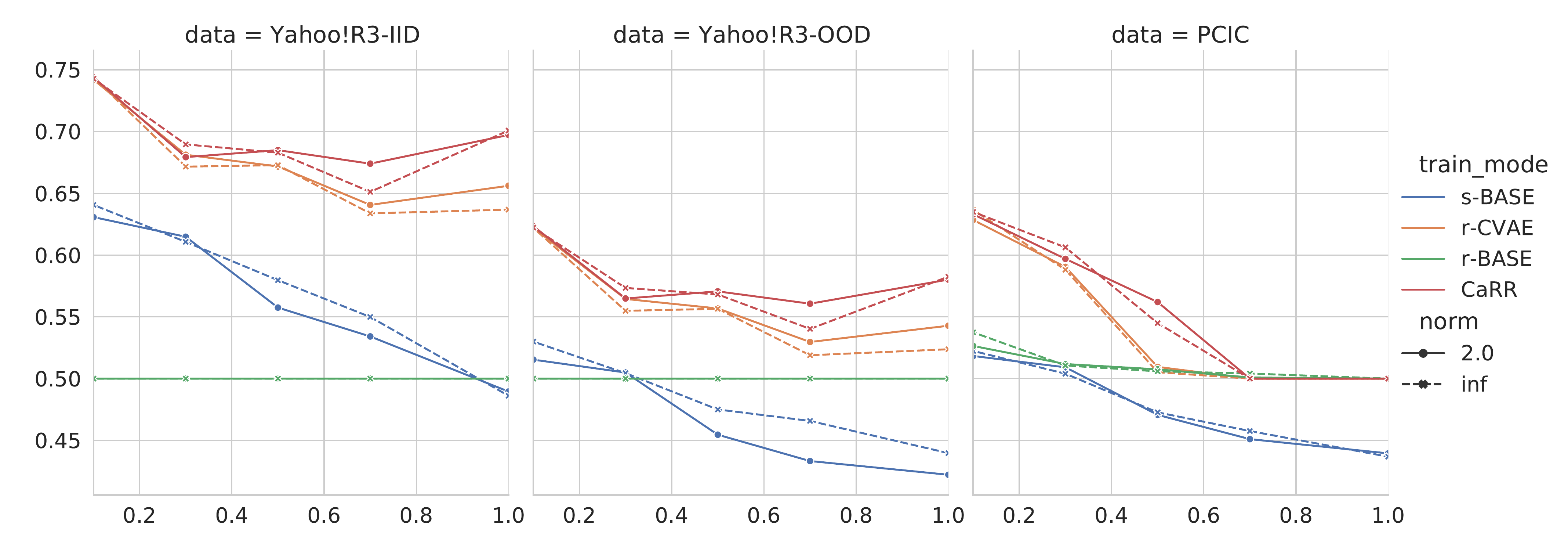}}
\caption{ Results under  different adversarial perturbations $\beta$ on three datasets. Axis-x is the attack degree $\beta$. Axis-y is the adv-AUC under attacked test datasets.}
\label{fig:epsilon_result}
\end{center}
\end{figure*}

        

\begin{table*}[h]
    \center
\caption{Additional overall results with standard error.}
\renewcommand\arraystretch{1.1}
\setlength{\tabcolsep}{3.3pt}
\begin{threeparttable}
\scalebox{.95}{
    \begin{tabular}{c|c|c|c|cc|cc|cc|cc}
    \hline\hline
dataset &     &     &      & AUC    & std    & ACC    & std    & adv\_AUC & std    & adv\_ACC & std    \\
\hline
\multirow{8}{*}{PCIC} & \multirow{4}{*}{standard} & \multirow{2}{*}{p=2}   & CaRR & 0.6416 & 0.0078 & 0.6803 & 0.0014 & 0.619    & 0.004  & 0.6625   & 0.0041 \\
         & &     & r-CVAE & 0.6328 & 0.0023 & 0.6725 & 0.0042 & 0.5893   & 0.0419 & 0.6429   & 0.0201 \\
         & &  \multirow{2}{*}{p=$\infty$} & CaRR & 0.6447 & 0.0041 & 0.6817 & 0.0043 & 0.6148   & 0.011  & 0.664    & 0.0104 \\
         & &     & r-CVAE & 0.6358 & 0.014  & 0.6779 & 0.0066 & 0.6138   & 0.0062 & 0.6601   & 0.0048 \\
&\multirow{4}{*}{robust}   & \multirow{2}{*}{p=2}   & CaRR & 0.6363 & 0.0045 & 0.6709 & 0.0042 & 0.6332   & 0.0024 & 0.6576   & 0.0006 \\
         & &     & r-CVAE & 0.63   & 0.0075 & 0.674  & 0.0069 & 0.6187   & 0.0051 & 0.6493   & 0.0013 \\
         & & \multirow{2}{*}{p=$\infty$} & CaRR & 0.639  & 0.007  & 0.6761 & 0.0024 & 0.6225   & 0.0057 & 0.6638   & 0.001  \\
         & &     & r-CVAE & 0.6363 & 0.0066 & 0.6733 & 0.0058 & 0.6088   & 0.0098 & 0.6596   & 0.0124\\
         \hline
         \multirow{8}{*}{Yahoo!R3 OOD}&\multirow{4}{*}{standard} & \multirow{2}{*}{p=2}   & CaRR & 0.6276 & 0.0001 & 0.6255 & 0.0022 & 0.5917 & 0.0071 & 0.5917 & 0.0072 \\
         & &     & r-CVAE & 0.6233 & 0.0005 & 0.6243 & 0.002  & 0.5865 & 0.0022 & 0.5872 & 0.0025 \\
         
         & &  \multirow{2}{*}{p=$\infty$} & CaRR & 0.629  & 0.0011 & 0.6257 & 0.0002 & 0.5966 & 0.0049 & 0.5965 & 0.0042 \\
         & &     & r-CVAE & 0.6253 & 0.0023 & 0.6249 & 0.0014 & 0.5855 & 0.0016 & 0.5863 & 0.0019 \\
& \multirow{4}{*}{robust}   & \multirow{2}{*}{p=2}   & CaRR & 0.6242 & 0.0009 & 0.6307 & 0.0012 & 0.6008 & 0.0009 & 0.601  & 0.0016  \\
         & &     & r-CVAE & 0.6191 & 0.0013 & 0.6241 & 0.0051 & 0.5882 & 0.0014 & 0.5907 & 0.0009 \\
         & & \multirow{2}{*}{p=$\infty$} & CaRR & 0.6238 & 0.0011 & 0.6284 & 0.0017 & 0.5993 & 0.0019 & 0.5999 & 0.0026  \\
         & &     & r-CVAE & 0.6186 & 0.001  & 0.6235 & 0.0028 & 0.5886 & 0.0014 & 0.5912 & 0.0012\\
         \hline
         \multirow{8}{*}{Yahoo!R3 i.i.d.}&\multirow{4}{*}{standard} & \multirow{2}{*}{p=2}   & CaRR & 0.7493 & 0.0004 & 0.7495 & 0.0015 & 0.7188 & 0.0015 & 0.7072 & 0.0013 \\
         & &     & r-CVAE & 0.7487 & 0.0001 & 0.7529 & 0.0027 & 0.7202 & 0.0029 & 0.7099 & 0.0027 \\
         
         & &  \multirow{2}{*}{p=$\infty$} & CaRR & 0.7497 & 0.0004 & 0.7503 & 0.0019 & 0.7191 & 0.0023 & 0.7099 & 0.0026 \\
         & &     & r-CVAE & 0.7488 & 0.0001 & 0.7515 & 0.0008 & 0.7191 & 0.0021 & 0.7072 & 0.0015 \\
& \multirow{4}{*}{robust}   & \multirow{2}{*}{p=2}   & CaRR & 0.7374 & 0.0024 & 0.7158 & 0.0061 & 0.7247 & 0.0026 & 0.7159 & 0.0036  \\
         & &     & r-CVAE & 0.7376 & 0.0018 & 0.7151 & 0.0045 & 0.7194 & 0.0020 & 0.7082 & 0.0021 \\
         & & \multirow{2}{*}{p=$\infty$} & CaRR & 0.7378 & 0.0015 & 0.7168 & 0.0015 & 0.7210 & 0.0031 & 0.7107 & 0.0040  \\
         & &     & r-CVAE & 0.7341 & 0.0007 & 0.7093 & 0.0035 & 0.7180 & 0.0017 & 0.7080 & 0.0016\\
         \hline
         \multirow{8}{*}{Coat OOD}&\multirow{4}{*}{standard} & \multirow{2}{*}{p=2}   & CaRR & 0.5725 & 0.0005 & 0.5732 & 0.0005 & 0.5608 & 0.0003 & 0.5601 & 0.0004 \\
         & &     & r-CVAE & 0.5671 & 0.0005 & 0.5649 & 0.0006 & 0.5586 & 0.0002 & 0.554  & 0.0001 \\
         
         & &  \multirow{2}{*}{p=$\infty$} & CaRR & 0.5705 & 0.0013 & 0.5718 & 0.0017 & 0.5643 & 0.0001 & 0.5659 & 0.0006 \\
         & &     & r-CVAE & 0.5656 & 0.0005 & 0.5643 & 0.0007 & 0.5527 & 0.0074 & 0.5478 & 0.0081 \\
& \multirow{4}{*}{robust}   & \multirow{2}{*}{p=2}   & CaRR & 0.5705 & 0.0015 & 0.5675 & 0.0015 & 0.5674 & 0.0002 & 0.565  & 0.0012  \\
         & &     & r-CVAE & 0.5634 & 0.0014 & 0.5591 & 0.0018 & 0.5572 & 0.0009 & 0.5522 & 0.0003 \\
         & & \multirow{2}{*}{p=$\infty$} & CaRR & 0.5707 & 0.0017 & 0.5681 & 0.0024 & 0.5653 & 0.0019 & 0.5659 & 0.0011  \\
         & &     & r-CVAE & 0.5629 & 0.0017 & 0.5586 & 0.0028 & 0.559  & 0.0004 & 0.5544 & 0.0007\\
         \hline
         \multirow{8}{*}{Coat i.i.d.}&\multirow{4}{*}{standard} & \multirow{2}{*}{p=2}   & CaRR & 0.7248 & 0.0011 & 0.7305 & 0.0016 & 0.7069 & 0.0023 & 0.7125 & 0.0036 \\
         & &     & r-CVAE & 0.7129 & 0.0009 & 0.7206 & 0.0022 & 0.7023 & 0.0041 & 0.7059 & 0.0061 \\
         
         & &  \multirow{2}{*}{p=$\infty$} & CaRR & 0.7283 & 0.0013 & 0.7355 & 0.0015 & 0.7125 & 0.0007 & 0.7196 & 0.001 \\
         & &     & r-CVAE & 0.7106 & 0.0029 & 0.7184 & 0.0033 & 0.7029 & 0.0008 & 0.7106 & 0.0094 \\
& \multirow{4}{*}{robust}   & \multirow{2}{*}{p=2}   & CaRR & 0.7265 & 0.0032 & 0.7331 & 0.0027 & 0.7196 & 0.0046 & 0.7261 & 0.0042  \\
         & &     & r-CVAE & 0.7087 & 0.0005 & 0.7169 & 0.0016 & 0.7058 & 0.002  & 0.7141 & 0.0036 \\
         & & \multirow{2}{*}{p=$\infty$} & CaRR & 0.7276 & 0.0028 & 0.7339 & 0.002  & 0.7208 & 0.0023 & 0.727  & 0.0019  \\
         & &     & r-CVAE & 0.7147 & 0.0023 & 0.7222 & 0.0026 & 0.7105 & 0.0039 & 0.7181 & 0.0043 \\
         \hline\hline
    \end{tabular}
}   
\end{threeparttable}    
\label{tab:pcic_std}

\end{table*}

\begin{figure*}[h]
\begin{center}
\centerline{\includegraphics[width=1\columnwidth]{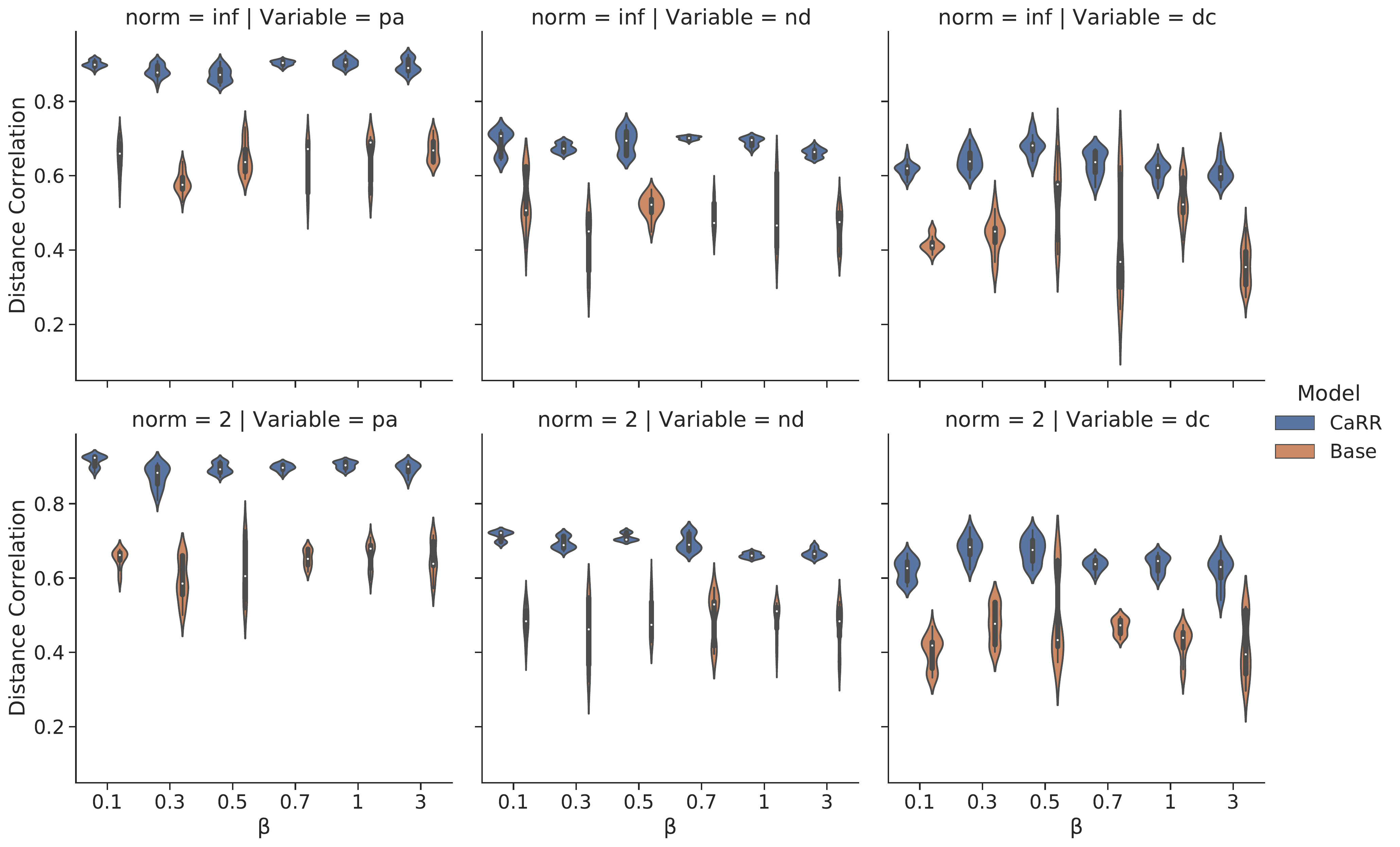}}
\caption{Identify results on synthetic dataset over different range of $\beta$ under robust training.}
\label{fig:6}
\end{center}
\end{figure*}

\begin{figure*}[h]
\begin{center}
\centerline{\includegraphics[width=1\columnwidth]{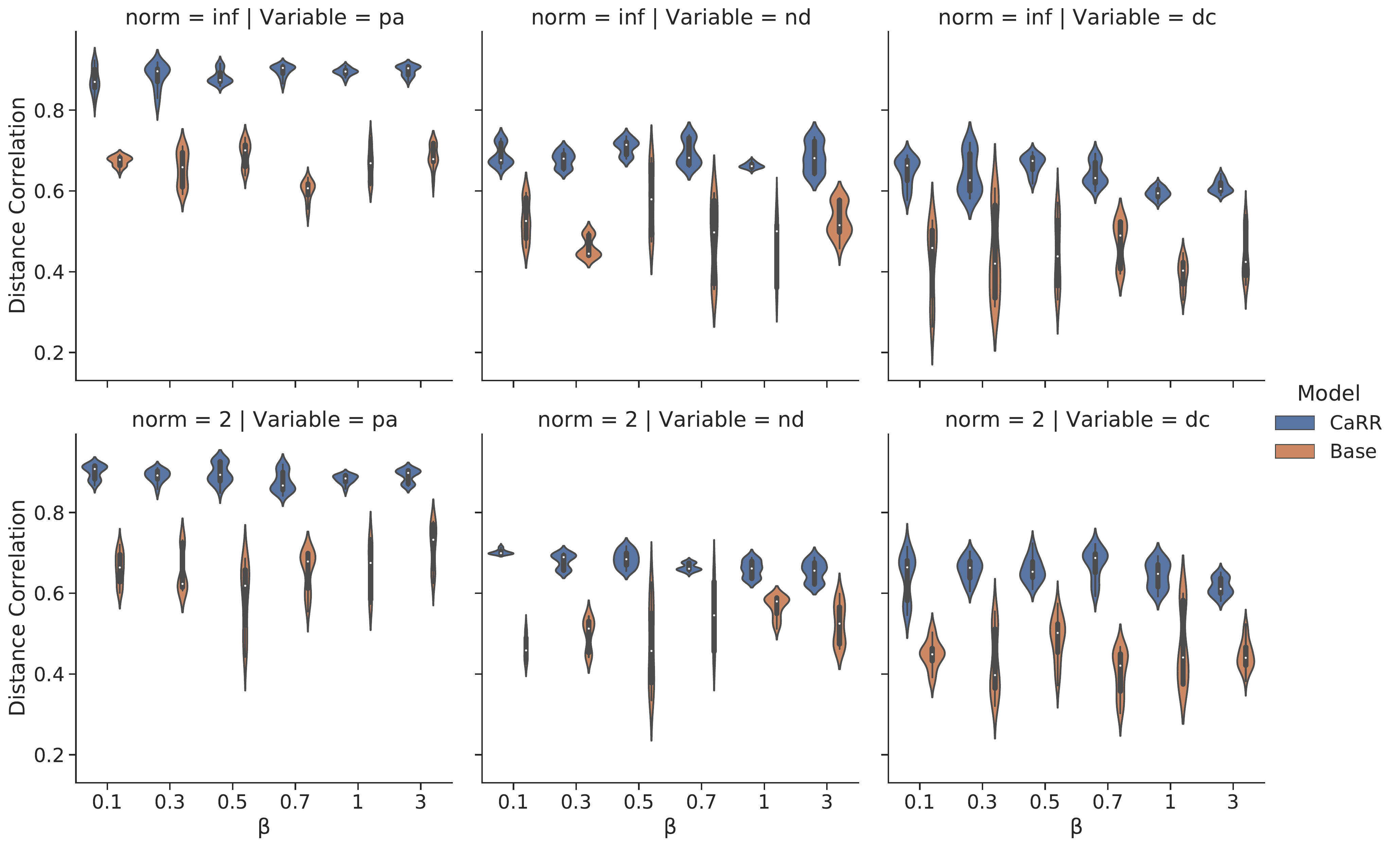}}
\caption{Identify results on synthetic dataset over different range of $\beta$ under standard training.}
\label{fig:7}
\end{center}
\end{figure*}

\begin{figure*}[h]
\begin{center}
\centerline{\includegraphics[width=1\columnwidth]{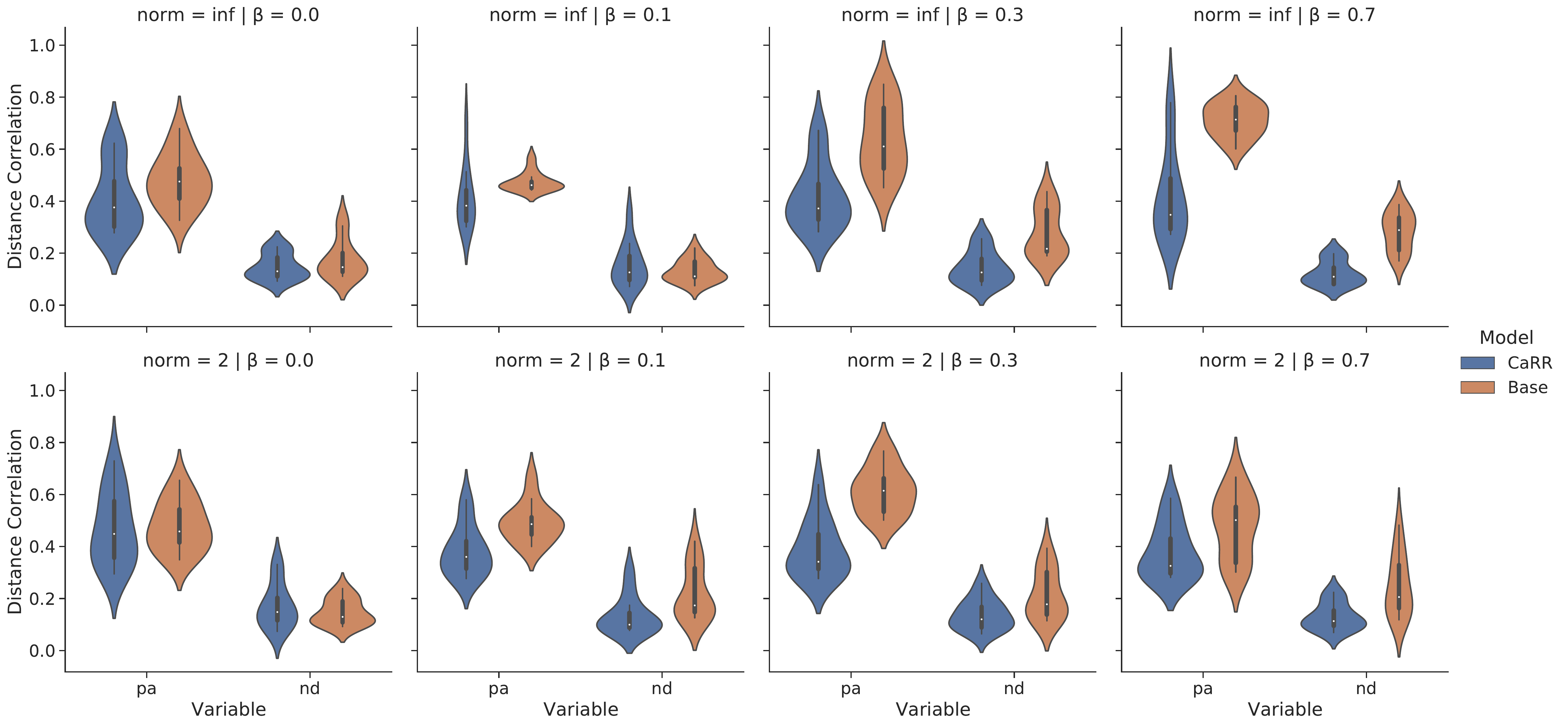}}
\caption{Identify results on CelebA-anno dataset over different range of $\beta$ during early optimization step.}
\label{fig:8}
\end{center}
\end{figure*}

\begin{figure*}[h]
\begin{center}
\centerline{\includegraphics[width=1\columnwidth]{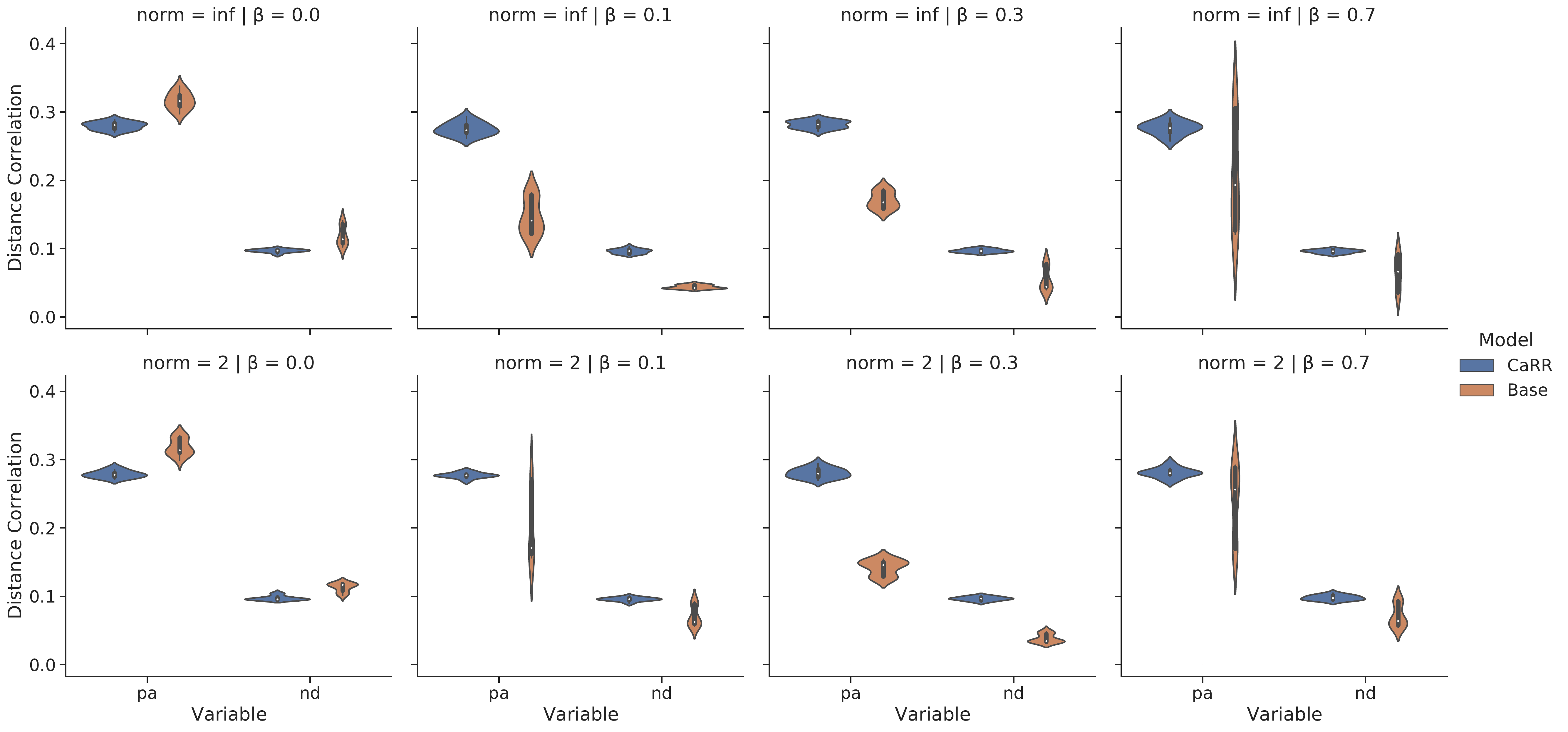}}
\caption{Identify results on CelebA-anno dataset over different range of $\beta$ after converging.}
\label{fig:9}
\end{center}
\end{figure*}
\end{document}